%% file: main.tex
\documentclass[nohyperref]{article}

\usepackage{microtype}
\usepackage{graphicx}
\usepackage{subfigure}
\usepackage{booktabs} 

\usepackage{titletoc}
\usepackage[page,header]{appendix}

\newcommand{\alglinelabel}{%
  \addtocounter{ALC@line}{-1}
  \refstepcounter{ALC@line}
  \label
}

\usepackage[accepted,nohyperref]{icml2023}
\usepackage{hyperref}
\definecolor{mydarkblue}{rgb}{0,0.08,0.45}
\hypersetup{ %
pdftitle={},
pdfauthor={},
pdfsubject={Proceedings of the International Conference on Machine Learning 2023},
pdfkeywords={},
pdfborder=0 0 0,
pdfpagemode=UseNone,
colorlinks=true,
linkcolor=mydarkblue,
citecolor=mydarkblue,
filecolor=mydarkblue,
urlcolor=mydarkblue,
}

\ifdefined\isaccepted \else
\hypersetup{pdfauthor={Anonymous Submission}}
\fi

\usepackage[accepted]{icml2023}

\usepackage{amsmath}
\usepackage{amssymb}
\usepackage{mathtools}
\usepackage{amsthm}

\usepackage[capitalize,noabbrev]{cleveref}

\theoremstyle{plain}
\newtheorem{theorem}{Theorem}[section]

\newtheorem{lemma}[theorem]{Lemma}
\newtheorem{corollary}[theorem]{Corollary}
\theoremstyle{definition}
\newtheorem{definition}[theorem]{Definition}

\theoremstyle{remark}
\newtheorem{remark}[theorem]{Remark}

\usepackage[textsize=tiny]{todonotes}

\newcommand{\E}{\operatornamewithlimits{\mathbb{E}}}

\newcommand{\argmax}{\operatornamewithlimits{\mathrm{argmax}}}
\newcommand{\argmin}{\operatornamewithlimits{\mathrm{argmin}}}
\newcommand{\tr}{\operatorname{\mathrm{Tr}}}
\renewcommand{\O}{\operatorname{\mathcal O}}
\newcommand{\Otil}{\operatorname{\tilde{\mathcal O}}}
\renewcommand{\P}{\operatorname{\mathbb P}}
\newcommand{\trans}{\mathsf{T}}
\newcommand{\mS}{\mathcal{S}}
\newcommand{\mA}{\mathcal{A}}

\newcommand{\mT}{\mathcal{T}}

\renewcommand{\tilde}{\widetilde}
\renewcommand{\hat}{\widehat}
\renewcommand{\bar}{\overline}

\usepackage{nicefrac}
\usepackage{multirow}
\usepackage{footnote}
\usepackage{enumitem}
\setitemize{leftmargin=*, nosep}
\setenumerate{leftmargin=*, nosep}
\usepackage{comment}
\usepackage{bbm}
\usepackage{bm}
\allowdisplaybreaks

\renewcommand{\algorithmicrequire}{\textbf{Input:}}
\renewcommand{\algorithmicensure}{\textbf{Output:}}

\Crefname{ALC@line}{Line}{Lines}
\Crefformat{equation}{Eq. #2(#1)#3}
\Crefrangeformat{equation}{Eqs. #3(#1)#4 to #5(#2)#6}
\Crefmultiformat{equation}{Eqs. #2(#1)#3}{ and #2(#1)#3}{, #2(#1)#3}{ and #2(#1)#3}
\Crefrangemultiformat{equation}{Eqs. #3(#1)#4 to #5(#2)#6}{ and #3(#1)#4 to #5(#2)#6}{, #3(#1)#4 to #5(#2)#6}{ and #3(#1)#4 to #5(#2)#6}

\newcommand\scalemath[2]{\scalebox{#1}{\mbox{\ensuremath{\displaystyle #2}}}}

\begin{document}

\twocolumn[
\icmltitle{Refined Regret for Adversarial MDPs with Linear Function Approximation}



\icmlsetsymbol{equal}{*}

\begin{icmlauthorlist}
\icmlauthor{Yan Dai}{iiis}
\icmlauthor{Haipeng Luo}{usc}
\icmlauthor{Chen-Yu Wei}{mit}
\icmlauthor{Julian Zimmert}{google}
\end{icmlauthorlist}

\icmlaffiliation{iiis}{IIIS, Tsinghua University}
\icmlaffiliation{usc}{University of Southern California}
\icmlaffiliation{mit}{IDSS, MIT}
\icmlaffiliation{google}{Google Research}

\icmlcorrespondingauthor{Yan Dai}{yan-dai20@mails.tsinghua.edu.cn}
\icmlcorrespondingauthor{Haipeng Luo}{haipengl@usc.edu}
\icmlcorrespondingauthor{Chen-Yu Wei}{chenyuw@mit.edu}
\icmlcorrespondingauthor{Julian Zimmert}{zimmert@google.com}

\icmlkeywords{Machine Learning, ICML}

\vskip 0.3in
]



\printAffiliationsAndNotice{}  

\begin{abstract}
We consider learning in an adversarial Markov Decision Process (MDP) where the loss functions can change arbitrarily over $K$ episodes and the state space can be arbitrarily large. We assume that the Q-function of any policy is linear in some known features, that is, a linear function approximation exists.
The best existing regret upper bound for this setting \citep{luo2021policy_nipsver} is of order $\Otil(K^{2/3})$ (omitting all other dependencies), given access to a simulator.
This paper provides two algorithms that improve the regret to $\Otil(\sqrt K)$ in the same setting.
Our first algorithm makes use of a refined analysis of the Follow-the-Regularized-Leader (FTRL) algorithm with the log-barrier regularizer. This analysis allows the loss estimators to be arbitrarily negative and might be of independent interest.
Our second algorithm develops a magnitude-reduced loss estimator, further removing the polynomial dependency on the number of actions in the first algorithm and leading to the optimal regret bound (up to logarithmic terms and dependency on the horizon).
Moreover, we also extend the first algorithm to simulator-free linear MDPs, which achieves $\Otil(K^{8/9})$ regret and greatly improves over the best existing  bound $\Otil(K^{14/15})$. This algorithm relies on a better alternative to the Matrix Geometric Resampling procedure by~\citet{neu2020efficient}, which could again be of independent interest.
\end{abstract}

\section{Introduction}
\input{introduction}

\subsection{Related Work}
\input{related-work}

\section{Preliminaries}
\input{setup}

\subsection{Dilated Bonuses for Policy Optimization}\label{sec:dialted_bonuses}
\input{dilated-bonus}

\section{$\Otil(\sqrt K)$ Regret for Linear-Q MDPs via Refined Log-Barrier Analysis}
\label{sec:linear-q log-barrier}
\input{linear-q-log-barrier}

\section{Removing Polynomial Dependency on $A$ via a New Magnitude-Reduced Estimator}
\label{sec:linear-q variance-reduced}
\input{linear-q-variance-reduced}
\section{Generalizing to Linear MDPs}
\label{sec:linear-MDP without MGR}
\input{linear-mdp}

\section{Conclusion}
\input{conclusion}

\section*{Acknowledgements}
We thank the anonymous reviewers for their insightful comments.
HL is supported by NSF Award IIS-1943607 and a Google Research Scholar Award.

\bibliography{references}
\bibliographystyle{icml2023}

\onecolumn
\newpage
\appendix
\renewcommand{\appendixpagename}{\centering \LARGE Supplementary Materials}
\appendixpage


\section{Auxiliary Lemmas and Omitted Algorithms}
\input{appendix-algorithms}

\newpage
\section{Omitted Proofs in \Cref{sec:linear-q log-barrier} (Linear-Q Algorithm Using Log-Barrier Regularizers)}
\label{sec:appendix linear-q log-barrier}
\input{appendix-linear-q-log-barrier}

\newpage
\section{Omitted Proofs in \Cref{sec:linear-q variance-reduced} (Linear-Q Algorithm Using Magnitude-Reduced Estimators)}
\label{sec:appendix linear-q variance-reduced}
\input{appendix-linear-q-variance-reduced}

\input{alg-linear-mdp}
\newpage
\section{Omitted Proofs in \Cref{sec:linear-MDP without MGR} (Linear MDP Algorithm)}
\label{sec:appendix linear-mdp}
\input{appendix-linear-mdp}

\end{document}

%% file: introduction.tex

Markov Decision Processes (MDPs) have been widely used to model reinforcement learning problems, where an agent needs to make decisions sequentially and to learn from the feedback received from the environment.
In this paper, we focus on adversarial MDPs where the loss functions can vary with time and the state space can also be arbitrarily large, capturing the fact that in real-world applications such as robotics, the environment can be non-stationary and the number of states can be prohibitively large.

To handle large state spaces, one of the most common methods in the literature is to assume a linear-function approximation \citep{yang2020reinforcement,jin2020provably,wei2021learning,zanette2021cautiously,neu2021online,luo2021policy_nipsver}, where the expected loss suffered by any policy from any state-action pair, commonly known as the Q-function, is linear in a set of known features.

While such assumptions are commonly used in the literature, the minimax optimal regret attainable by the agent in adversarial environments is still poorly understood.
\footnote{To be more specific, we only consider bandit feedback in this paper, where the agent can only observe her experienced losses. For the easier full-information setting, $\sqrt K$-style near-optimal regret has already been achieved~\citep{he2022near} (\textit{cf.} \Cref{tab:table-regret}).\label{footnote:full information}}
Specifically, for adversarial linear-Q MDPs, the best existing bound is of order $K^{2/3}$ where $K$ is the number of episodes \citep{luo2021policy_nipsver}.
\footnote{Meanwhile, if a ``good'' exploratory policy $\pi_0$ (formalized in \Cref{footnote:good exploratory policy}) is granted, $\sqrt K$-style bounds are also achievable, though with some additional dependencies on the quality of $\pi_0$  \citep{luo2021policy_nipsver,neu2021online}; see \Cref{tab:table-regret}.}
In that paper, the authors assume a transition simulator (i.e., the agent is allowed to draw a trajectory starting from any state-action pair, sampled from the actual transition and a given policy, without any cost) and achieve $\Otil(d^{2/3} H^2 K^{2/3})$ regret where $H$ is the length of each episode and $d$ is the dimension of the feature space.
On the other hand, the best lower bound for this setting (induced from the special case of adversarial linear bandits) is of order $\Omega(\sqrt{K})$~\citep{dani2008price}.
Therefore, a natural question arises:

\textit{\textbf{Is it possible to design an algorithm in adversarial linear-Q MDPs that attains $\Otil(\sqrt K)$ regret bound?}}

\begin{table*}[htb]
\begin{minipage}{\textwidth}
\caption{Overview of Our Results and Comparisons with the Most Related Works}
\label{tab:table-regret}
\centering
\begin{savenotes}
\renewcommand{\arraystretch}{1.6}
\resizebox{\textwidth}{!}{%
\begin{tabular}{|c|c|c|c|c|}\hline
Algorithm & Setting\footnote{Linear MDP is a special case of linear-Q MDP, while linear
mixture MDP is generally incomparable to these two.} 
& Transition & Assumption\footnote{The definitions of ``exploratory policy'' and $\lambda_0$ are stated in \Cref{footnote:good exploratory policy}, while ``full information'' means the entire loss function is revealed at the end of each episode (which is easier than our bandit-feedback setting).} & Regret \\\hline
\multirow{2}{*}{\shortstack{\textsc{Dilated Bonus}\\\citep{luo2021policy_nipsver}}} & \multirow{4}{*}{\shortstack{Linear-Q MDP\\\\(\Cref{def: linear Q})}} & \multirow{4}{*}{\shortstack{Simulator\\\\ (\Cref{def:simulator})}} & None & $\Otil(d^{2/3}H^2 \bm{K^{2/3}})$ \\\cline{4-5}
 & & & Exploratory Policy & $\Otil\big (\text{poly}(d,H)\bm{(K/\lambda_0)^{1/2}}\big )$ \\\cline{1-1} \cline{4-5} 
\Cref{alg:linear-q with log-barrier} (\textbf{This work}) & & & \multirow{2}{*}{\textbf{None}} & $\Otil(A^{1/2} d^{1/2} H^3 \bm{K^{1/2}})$ \\\cline{1-1} \cline{5-5}
\Cref{alg:linear-q using variance-reduced} (\textbf{This work}) & & &  & $\Otil(d^{1/2}H^3\bm{K^{1/2}})$ \\\hline
\textsc{POWERS} \citep{he2022near} & Linear Mixture MDP & Unknown & Full Information & $\Otil(dH \bm{K^{1/2}})$ \\\hline
\multirow{1}{*}{ \shortstack{\small \textsc{Online Q-REPS} \\ \small{\citep{neu2021online}}}} & \multirow{4}{*}{\shortstack{Linear MDP\\ \\(\Cref{def:linear MDP})}} & Known & Exploratory Policy & $\Otil\big (\text{poly}(d,H)\bm{(K/\lambda_0)^{1/2}}\big )$\footnote{This is a refined version of the original $\O(\sqrt{K\log K})$ bound presented in their Theorem 1, which contains no explicit dependency on $\lambda_0$ by assuming $K=\Omega(\exp(\lambda_0^{-1}))$. See Review hYYK at \url{https://openreview.net/forum?id=gviX23L1bqw}.
} \\\cline{1-1}\cline{3-5}
\multirow{2}{*}{\shortstack{\textsc{Dilated Bonus}\\\citep{luo2021policy,luo2021policy_nipsver}}} &  & \multirow{3}{*}{Unknown} & None & $\Otil(d^2H^4\bm{K^{14/15}})$ \\\cline{4-5}
& & & Exploratory Policy & $\Otil\big (\text{poly}(d,H)\bm{\big (K/\lambda_0^{2/3}\big )^{6/7}}\big )$ \\\cline{1-1} \cline{4-5}
\Cref{alg:linear MDP without MGR}  (\textbf{This work}) & & & \textbf{None} & $\Otil(H^{20/9}A^{1/9}d^{2/3}\bm{K^{8/9}})$ \\\hline
\end{tabular}}
\end{savenotes}
\end{minipage}
\end{table*}

In this work, we answer this question in the affirmative by developing two algorithms that both attain $\Otil(\sqrt K)$ regret in adversarial linear-Q MDPs when a simulator is granted, closing the gap with the lower bound. 
Both of our algorithms follow the same framework of the policy optimization algorithm with \textit{dilated exploration bonuses} of~\citep{luo2021policy_nipsver}, but with important modifications.
Specifically, our first algorithm applies Follow-the-Regularized-Leader (FTRL) with the log-barrier regularizer (instead of the negative entropy regularizer used by~\citet{luo2021policy_nipsver}) at each state,
 and we develop a new analysis inspired by \citet{zimmert2022return} 
that allows the loss estimators to be arbitrarily negative (in contrast, in the usual analyses of FTRL, the loss estimator cannot be too negative). This new analysis is the key to improving the regret to $\O(\sqrt K)$ and might be of independent interest even for multi-armed bandits.

Although our first algorithm is simple in design and analysis, the log-barrier regularizer causes an $\O(\sqrt A)$ factor in the regret bound $\Otil(H^3 \sqrt{AdK})$ where $A$ is the number of actions.
As a rescue, we develop another algorithm that uses the negative entropy regularizer with a \textit{magnitude-reduced} loss estimator. This algorithm attains an $\Otil(H^3\sqrt{dK})$ regret bound, which is optimal in $d$ and $K$ up to logarithmic factors and gets rid of the $\text{poly}(A)$ dependency.
We note that our algorithm is of interest even for the special case of adversarial linear bandits, as it removes the need of explicit John's exploration introduced by \citet{bubeck2012towards}.

At last, we also apply our method to \textit{simulator-free} linear MDPs (formally defined in \Cref{def:linear MDP}), yielding an efficient algorithm with $\Otil(K^{8/9})$ regret and greatly outperforming the best existing bound $\Otil(K^{14/15})$ \citep{luo2021policy}.\footnote{\citep{luo2021policy} is a refined version of \citep{luo2021policy_nipsver}.}
Remarkably, in this application, we not only use our refined analysis for FTRL with the log-barrier regularizer, but also develop a more sample-efficient alternative to the Matrix Geometric Resampling (MGR) method introduced by \citet{neu2020efficient} and later adopted by \citet{neu2021online} and \citet{luo2021policy}, which could also be of independent interest.

%% file: related-work.tex

\textbf{MDPs with Linear-Function Approximation.}
Linear function approximation has been a standard technique for handling large state spaces in RL, but only recently have researchers provided strong regret guarantees for these algorithms under precise conditions. 
\citet{yang2020reinforcement} introduced a linear function approximation scheme called embedded linear transition MDPs where the transition kernels are bilinear, i.e., the probability of reaching state $s'$ in the $h$-th step of an episode after taking action $a$ at state $s$ is $\P(s'\mid s,a)=\phi(s,a)^\trans M\psi(s')$ for some known feature mappings $\phi$ and $\psi$ and an unknown $M$.
\citet{jin2020provably} loosen the assumption to linear MDPs (\Cref{def:linear MDP}), i.e., $\P(s'\mid s,a)=\phi(s,a)^\trans \nu(s')$ where $\nu$ is unknown.
\citet{zhou2021provably} study the linear mixture MDP with $\P(s'\mid s,a)=\psi(s'\mid s,a)^\trans \theta$ where $\theta$ is unknown. This generalizes embedded linear transition MDPs but is incomparable with linear MDPs.
Another common model is linear-Q MDPs \citep{abbasi2019politex} where the Q-function with respect to any policy $\pi$ can be written as $Q_h^\pi(s,a)=\phi^\trans(s,a) \theta_h^\pi$ for some unknown $\theta_h^\pi$ (\Cref{def: linear Q}). 
Linear MDPs are special cases of linear-Q MDPs.

\textbf{Adversarial MDPs.}
MDPs with adversarial losses were first studied in the tabular cases where the state space has a small size $S\ll \infty$.
\citet{zimin2013online} first assume known transitions and achieve $\Otil(H\sqrt K)$ regret when full information is available and $\Otil(\sqrt{HSAK})$ regret when only bandit feedback is available.
\citet{rosenberg2019online} then study the unknown-transition case and get $\Otil(HS\sqrt{AK})$ regret with full information.
Finally, \citet{jin2020learning} tackle the hardest case with unknown transitions and bandit feedback and achieve $\Otil(HS\sqrt{AK})$ regret as well.

Results for adversarial MDPs with linear-function approximations are summarized in Table~\ref{tab:table-regret}. 
Specifically, \citet{cai2020provably} study unknown-transition, full-information linear mixture MDPs and get $\Otil(dH^{3/2} \sqrt K)$ regret; this result is further improved to $\Otil(dH\sqrt K)$ by \citet{he2022near}.
\citet{neu2021online} then study known-transition, bandit-feedback linear MDPs. Provided with a ``good'' exploratory policy,\footnote{Formally, a ``good'' exploratory policy $\pi_0$ ensures that $\lambda_{\min}(\Sigma_h^{\pi_0})\ge \lambda_0$ for all $h\in [H]$ and a positive constant $\lambda_0$, where $\Sigma_h^{\pi_0}$ means the covariance of $\pi_0$ in layer $h$ (see \Cref{eq:covariance definition}). \label{footnote:good exploratory policy}}
their algorithm achieves  an $\Otil(\text{poly}(d,H)\sqrt{K/\lambda_0})$ regret guarantee.
Later, \citet{luo2021policy_nipsver} study bandit-feedback linear-Q MDPs with a simulator, providing an $\Otil(d^{2/3} H^2 K^{2/3})$ regret bound. A refined version \citep{luo2021policy} considers simulator-free bandit-feedback linear MDPs, giving an $\Otil(d^2H^4K^{14/15})$ bound. Meanwhile, provided with a good exploratory policy (see \Cref{footnote:good exploratory policy}), these bounds improve to $\Otil(\text{poly}(d,H)\sqrt{K/\lambda_0})$ and $\Otil(\text{poly}(d,H)\lambda_0^{-4/7}K^{6/7})$, respectively.
As a final remark, we point out that exploration in MDPs with huge state spaces is challenging and assuming an exploratory policy is unrealistic  --- as far as we know, there are no results ensuring even the \textit{existence} of such a ``good'' exploratory policy, 
let alone finding it efficiently.
We emphasize that our results \textit{do not} require such unrealistic exploratory assumptions.

\textbf{Policy Optimization Algorithms.}
Policy optimization algorithms for RL directly optimize the learner's policy.
They are more resilient to model misspecification or even adversarial manipulation. But due to their local search nature, they suffer from the notorious \emph{distribution mismatch} issue. 
Recent theoretical works address this issue by exploration bonuses; see \citep{agarwal2020pc, shani2020optimistic, zanette2021cautiously} for stochastic settings and \citep{luo2021policy_nipsver} for adversarial settings. 
While the latter work achieves near-optimal regret in tabular settings, there is a huge room for improvement when considering function approximation.
Our work builds on top of their framework (especially the dilated bonus idea) and significantly improves their results in linear-function approximation settings. 


\textbf{Concurrent Works.}
Aside from this paper, there are several concurrent submissions also studying linear-Q or linear MDPs with adversarial losses, bandit feedback, and unknown transitions.
\citet{sherman2023improved} propose computationally efficient algorithms for linear MDPs: without simulators, their algorithm achieves $\Otil(K^{6/7})$ regret and outperforms our $\Otil(K^{8/9})$; however, when simulators are made available, their $\Otil(K^{2/3})$ result becomes worse than our $\Otil(\sqrt K)$ bound for linear-Q MDPs (recall that linear MDPs are special linear-Q MDPs), albeit being more computationally efficient.
\citet{kong2023improved} propose an inefficient algorithm for linear MDPs based on recent ideas of \citet{wagenmaker2022instance} and get $\Otil(K^{4/5})$ regret.
\citet{lancewicki2023delay} consider linear-Q MDPs with delayed feedback, which recovers the $\Otil(K^{2/3})$ bound by \citet{luo2021policy_nipsver} when there are no feedback delays.

%% file: setup.tex

\textbf{Notations.}
For $N\in \mathbb N$, $[N]$ denotes the set $\{1,2,\ldots,N\}$.
For a (possibly infinite) set $X$, we denote the probability simplex over $X$ by $\triangle(X)$.
For a random event $\mathcal E$, denote its indicator by $\mathbbm 1[\mathcal E]$.
For two square matrices $A,B$ of the same size, $\langle A,B\rangle$ stands for $\tr(A^\trans B)$.
For $x\in \mathbb R$, define $(x)_-$ as $\min\{x,0\}$. 
We use $\Otil$ to hide all logarithmic factors.

\textbf{No-Regret Learning in MDPs.} An (episodic) adversarial MDP is specified by a tuple $\mathcal M=(\mS,\mA,\P,\ell)$ where
$\mS$ is  the state space (possibly infinite),
$\mA$ is  the action space (assumed to be finite with size $A = \lvert \mA\rvert$),
$\mathbb P\colon \mathcal S\times \mathcal A\to \triangle(\mathcal S)$ is the transition,
and $\ell\colon [K]\times \mathcal S\times \mathcal A\to [0,1]$ is the loss function chosen arbitrarily by an adversary.
Following \citet{luo2021policy_nipsver}, the state space is assumed to be \textit{layered}, i.e., $\mS=\mS_1 \cup \mS_2 \cup \cdots \cup \mS_H$ where $\mS_{h}\cap \mS_{h'}=\varnothing$ for any $1\le h<h'\le H$, and transition is only possible from one layer to the next one, that is, $\P(s' \mid s,a) \neq 0$ only when $s \in \mS_h$ and $s' \in \mS_{h+1}$ for some $h < H$. We also assume that there is an initial state $s_1$ such that $\mS_1=\{s_1\}$.
Note that as the regret does not contain any dependency on the size of $\mS$, this assumption is made without loss of generality.

The game lasts for $K$ episodes, each with length $H$. For each episode $k$, the agent is initialized at state $s_1$. For each step $h\in [H]$, she chooses an action $a_h\in \mA$, suffers and observes the loss $\ell_{k}(s_h,a_h)$, and transits to a new state $s_{h+1}$ independently sampled from the transition $\P(\cdot \mid s_h,a_h)$.

A policy $\pi$ of the agent is a mapping from $\mS$ to $\triangle(\mA)$.
Let $\Pi$ be the set of all policies. For each episode $k\in [K]$, let $\pi_k\in \Pi$ be the policy deployed by the agent. Its expected loss is then indicated by $V^{\pi_k}_k(s_1)$, where the \textit{state-value function} (or V-function in short) $V^{\pi}_k(s_1)$ is defined as follows for any episode $k$ and policy $\pi$:
\begin{align*}
V^{\pi}_k(s_1)\triangleq &\E \left [\sum_{h=1}^H \ell_{k}(s_h,a_h)\middle \vert (s_h,a_h)\sim \pi,\forall h\in [H] \right ],
\end{align*}
with $(s_h,a_h)\sim \pi, \forall h\in [H]$ denoting a trajectory sampled from $\pi$. 
The agent aims to minimize the cumulative total loss collected in all episodes, or equivalently, the \textit{regret}:
\begin{definition}[Regret]
The regret of the agent is
\begin{equation*}
\mathcal R_K\triangleq \E\left [\sum_{k=1}^K V_k^{\pi_1}(s_1)\right ]- \sum_{k=1}^K V_k^{\pi^\ast}(s_1),
\end{equation*}
where the expectation is taken over the randomness of both the agent and the transition, and $\pi^\ast$ is the optimal policy in hindsight (i.e., $\pi^\ast \in \argmin_{\pi \in \Pi} \sum_{k=1}^K V_k^{\pi}(s_1)$).
\end{definition}

\subsection{Linear-Q MDP and Linear MDP}
A concept closely related to the V-function is the \textit{action-value function} (\textit{a.k.a.} Q-function), which denotes the expected loss suffered by a policy $\pi$ starting from a given state-action pair $(s,a)$. Formally, we define for all $(s,a)\in \mS\times \mA$:
\begin{align}
\scalemath{0.9}{Q_k^\pi(s,a) = \ell_k(s,a) + \mathbbm 1[s \notin \mS_H]\E_{\substack{s'\sim \P(\cdot\mid s,a)\\a'\sim \pi(\cdot\mid s')}}\left[Q_k^\pi(s',a')\right].   \label{eq: Q-function}}
\end{align}
%
We can then define linear-Q MDPs as follows.
\begin{definition}[{\citep[Assumption 1]{luo2021policy_nipsver}}]\label{def: linear Q}
In a \textit{linear-Q MDP}, each state-action pair $(s,a)$ is associated with a known feature $\phi(s,a)\in \mathbb R^d$ with $\lVert \phi(s,a)\rVert_2\le 1$.
Moreover, for any policy $\pi\in \Pi$, episode $k\in [K]$, and layer $h\in [H]$, there exists a (hidden) vector $\theta_{k,h}^\pi \in \mathbb R^d$ such that
\begin{equation*}
Q_{k}^\pi(s_h,a_h)=\phi(s_h,a_h)^\trans \theta_{k,h}^\pi,\quad \forall s_h\in \mS_h,a_h\in \mA.
\end{equation*}
We assume that $\lVert \theta_{k,h}^\pi\rVert_2\le \sqrt d H$ for all $k,h,\pi$.
\end{definition}

\Cref{def: linear Q} does not specify any specific structure on the transition, making learning extremely difficult.
Consequently, prior works all assume the availability of a simulator that allows us to sample from the hidden transition:
\begin{definition}[Simulator]\label{def:simulator}
A \textit{simulator} accepts a state-action pair $(s,a)\in \mS\times \mA$ and generates a next-state output sampled from the true transition, i.e., $s'\sim \P(\cdot\mid s,a)$.
\end{definition}

When a simulator is unavailable, we consider a special case called linear MDPs which further impose a linear structure on the transition, enabling the agent to learn:
\begin{definition}[{\citep[Assumption 3]{luo2021policy_nipsver}}]\label{def:linear MDP}
A \textit{linear MDP} is a linear-Q MDP that additionally satisfies the following property:
for any $h\in [H]$, the transition from layer $h$ to layer $h+1$ can be written as follows:
\begin{align*}
&\quad \P(s_{h+1}\mid s_h,a_h)=\langle \phi(s_h,a_h),\nu(s_{h+1})\rangle,\\
&\forall s_{h+1}\in \mS_{h+1},s_h\in \mS_h,a_h\in \mA.
\end{align*}
Here, the mapping $\nu\colon \mS\to \mathbb R^d$ is also unrevealed to the agent. We also assume that $\lVert \nu(s)\rVert_2\le \sqrt d$ for all $s\in \mS$.
\end{definition}

In MDPs with linear function approximation, another important quantity associated with each policy $\pi$ is its \textit{covariance matrix} at each layer $h\in [H]$, defined as follows:
\begin{equation}\label{eq:covariance definition}
\Sigma_h^\pi=\E_{(s_h,a_h)\sim \pi}[\phi(s_h,a_h) \phi(s_h,a_h)^\trans ],\quad \forall h\in [H].
\end{equation}

%% file: dilated-bonus.tex

We now briefly introduce the policy optimization method and the key dilated bonus idea of \citet{luo2021policy_nipsver}, upon which our algorithms are built.
The foundation of policy optimization is the performance difference lemma \citep{kakade2002approximately}, which asserts that the regret of the agent can be viewed as the weighted average of the regret of some local bandit problem over each state. Formally, we have
\begin{equation*}
\scalemath{0.9}{\mathcal R_K=\sum_{h=1}^H \E_{s_h\sim \pi^\ast}\left[\sum_{k=1}^K  \sum_{a} \left(\pi_k(a|s_h) -  \pi^\ast(a|s_h)\right) Q_k^{\pi_k}(s_h,a)\right],}
\end{equation*}
where the part inside the expectation is exactly the regret of a multi-armed bandit (MAB) problem at state $s_h$ with ``loss'' $Q_k^{\pi_k}(s_h,a)$ (instead of $\ell_k(s_h,a)$) for action $a$.
Policy optimization algorithms then naturally run a bandit algorithm at each state with an appropriate Q-function estimator to learn the best policy directly. 
For example, in linear-Q MDPs, since the ``loss" $Q_k^{\pi_k}(s_h,a)$ is linear in some feature, it suggests running an adversarial linear bandit algorithm such as \textsc{Exp2}~\citep{bubeck2012towards} at each state.

However, as discussed in detail by~\citet{luo2021policy_nipsver}, the bias/variance of the Q-function estimator often leads to a regret term of the form $\sum_{k,h} \E_{(s_h,a_h)\sim  \pi^\ast} [b_k(s_h,a_h)]$ for some non-negative functions $b_1, \ldots, b_K \colon \mS\times \mA\to \mathbb R_{\ge 0}$, where $b_k(s,a)$ is often prohibitively large if $(s,a)$ is rarely visited by the agent. Hence, the expectation over $(s_h,a_h)\sim  \pi^\ast$ could be potentially large as well, while the expectation over $(s_h,a_h)\sim \pi_k$ is relatively small. This is the well-known \textit{distribution mismatch} issue:
for example, when applying a linear bandit algorithm at each state, $b_k(s_h,a_h)$ is roughly $ \beta \lVert \phi(s_h,a_h)\rVert_{(\Sigma_h^{\pi_k})^{-1}}^2$ for  some $\beta > 0$. 
Thus, $\E_{(s_h,a_h)\sim  \pi_k} [b_k(s_h,a_h)] =
\beta \big \langle (\Sigma_h^{\pi_k})^{-1},$ $\E_{(s_h,a_h)\sim \pi_k}[\phi(s_h,a_h) \phi(s_h,a_h)^\trans] \big \rangle= \beta d$; on the other hand, its counterpart with $(s_h,a_h)$ drawn from $\pi^\ast$ (i.e., $\E_{(s_h,a_h)\sim \pi^\ast}[b_k(s_h,a_h)]$) could be arbitrarily large.

To address this distribution mismatch issue and ``convert'' the measure from $\pi_k$ to $\pi_\ast$, \citet{luo2021policy_nipsver} consider treating these functions as exploration bonuses and further propose the so-called ``\textit{dilated bonus}'' functions $B_k(s,a)$:
\begin{align}
B_k(s,a)&=b_k(s,a) +{}\label{eq: dilated bonus}\\
&\quad \mathbbm 1[s \notin \mS_H] \big (1+\tfrac{1}{H}\big ) \E_{\substack{s'\sim \P(\cdot\mid s,a)\\a'\sim \pi_k(\cdot\mid s')}}\left[B_k(s',a')\right]. \nonumber
\end{align}
Compared to Eq.~\eqref{eq: Q-function}, $B_k$ can be viewed the Q-function of $b_k$, except that it assigns slightly more weight to deeper layers via the extra weighting $1+\frac{1}{H}$ to encourage more exploration to those layers.
Their algorithms then try to minimize the regret with respect to the ``optimistic" loss $Q_k^{\pi_k} - B_k$, instead of just $Q_k^{\pi_k}$, at each state.
Their analysis relies on the following key lemma, which shows that if the regret w.r.t. $Q_k^{\pi_k}-B_k$ at each state is in some particular form, then the distribution mismatch issue can be resolved: that is, the final regret $\mathcal R_K$ is in terms of $\sum_{k,h} \E_{(s_h,a_h)\sim {\color{blue} \pi_k}} [b_k(s_h,a_h)]$ instead of $\sum_{k,h} \E_{(s_h,a_h)\sim {\color{blue} \pi^\ast}} [b_k(s_h,a_h)]$. 
\begin{lemma}[Lemma 3.1 by \citet{luo2021policy}]\label{lem:dilated bonus}
If $\{b_k\}_{k=1}^K$ are non-negative, $B_k(s,a)$ is defined as in \Cref{eq: dilated bonus}, and the following holds for any state $s$:
\[
\scalemath{0.9}{ 
\begin{aligned}
&\hspace{-0.2cm}\quad \E\left[\sum_{k=1}^K  \sum_{a} \left(\pi_k(a|s) -  \pi^\ast(a|s)\right) \left(Q_k^{\pi_k}(s,a) -  B_k(s,a)\right)\right]  \nonumber \\
&\hspace{-0.2cm}\le X(s) + \E\left[\sum_{k=1}^K \sum_{a\in \mA} \left(\pi^\ast(a|s)b _k(s,a)+\frac{1}{H} \pi_k(a|s)B_k(s,a)\right)\right],
\end{aligned}}
\]
where $X(s)$ is an arbitrary function, then
\begin{equation*}
\hspace{-0.1cm}\mathcal R_K\le \sum_{h=1}^H\E_{s_h\sim \pi^\ast} [X(s_h)] ~+~ 3\sum_{k=1}^K \sum_{h=1}^H \E_{(s_h,a_h)\sim \pi_k} [b_k(s_h,a_h)].
\end{equation*}
\end{lemma}
\citet{luo2021policy_nipsver} show that the per-state regret bound in the condition of \Cref{lem:dilated bonus} indeed holds when deploying Follow-the-Regularized-Leader (FTRL) with the \textit{negative entropy} regularizer at each state.
However, for technical issues (which we will discuss and resolve in \Cref{sec:linear-q log-barrier,sec:linear-q variance-reduced}), they only achieve $\Otil(K^{2/3})$ regret in linear-Q MDPs.


%% file: linear-q-log-barrier.tex

As mentioned, our first algorithm replaces the negative entropy regularizer in FTRL used by \citet{luo2021policy_nipsver} with another regularizer called the \textit{log-barrier}.
Specifically, FTRL applied to the action space $\mA$ maintains a sequence of distributions $x_1,x_2,\ldots,x_T\in \triangle([A])$ via the FTRL update
\begin{equation*}
x_t=\argmin_{x\in \triangle([A])} \left \{\eta \left \langle x,\sum_{\tau<t}c_\tau\right \rangle+\Psi(x)\right \},
\end{equation*}
where $c_1, \ldots, c_T \in \mathbb R^A$ is the loss sequence, $\Psi: \triangle([A]) \rightarrow \mathbb R$ is the regularizer, and $\eta > 0$ is the learning rate.
The classical MAB algorithm \textsc{Exp3}~\citep{auer2002nonstochastic} and its linear-bandit variant \textsc{Exp2}~\citep{bubeck2012towards} both use the negative entropy regularizer $\Psi(x)=\sum_{i=1}^A p_i\ln p_i$.

Starting from~\citep{foster2016learning}, a sequence of works discover many nice properties of a different regularizer called log-barrier, defined as $\Psi(p)=\sum_{i=1}^A \ln \frac{1}{p_i}$.
Here, we present yet another new and useful property of log-barrier, summarized in the following lemma.%
\footnote{When revising this manuscript, we found that a key property we use when proving this lemma (see our \Cref{eq:key property of log-barrier}) was also independently developed by \citet[Corollary 7]{putta2022scale} under the name ``new local-norm lower-bounds for Bregman divergences''. They used this property to handle scale-free adversarial Multi-Armed Bandits (MABs) where the losses are not always constantly bounded but can be arbitrarily positive or negative.}
Our proof is inspired by the analysis of the log-determinant regularizer by~\citet{zimmert2022return} and is deferred to \Cref{sec:log-barrier lemma proof}.

\begin{lemma}\label{lem:log-barrier regret bound}
Let $x_1,\ldots,x_T\in  \triangle([A])$ be defined as
\begin{equation*}
x_{t}=\argmin_{x\in \triangle([A])}\bigg \{\eta\bigg\langle x,\sum_{\tau<t}c_\tau\bigg\rangle+\Psi(x)\bigg \},\quad \forall t=1,\ldots, T,
\end{equation*}
where $c_t\in \mathbb R^A$ is an arbitrary loss vector corresponding to the $t$-th iteration and $\Psi(p)=\sum_{i=1}^A \ln \frac{1}{p_i}$ is the log-barrier regularizer.
Then the regret against any distribution $y\in \triangle([A])$ with respect to $\{c_t\}_{t=1}^T$ is bounded as:
\begin{align*}
\sum_{t=1}^T\langle x_t-y,c_t\rangle\
\le \frac{\Psi(y)-\Psi(x_1)}{\eta}+\eta \sum_{t=1}^T\sum_{i=1}^A x_{t,i}c_{t,i}^2.
\end{align*}
\end{lemma}

Readers familiar with the \textsc{Exp2}/\textsc{Exp3} analysis would immediately recognize this regret bound since it is also the same bound that FTRL with negative entropy (also known as Hedge) enjoys.
However, the key distinction is that for log-barrier, this holds without \textit{any} requirement on the magnitude of $c_{t,i}$, while for negative entropy, one must require $\eta c_{t,i}\ge -1$ for all $t\in [T]$ and $i\in [A]$ (i.e., losses cannot be too negative; see \Cref{lem:hedge lemma} in the appendix for more details).
This turns out to be critical for improving the regret when applying it to linear-Q MDPs, as discussed later.

\vspace{2pt}
\begin{remark}\label{remark:log-barrier lemma is powerful}
Our \Cref{lem:log-barrier regret bound} also answers the open question raised by \citet{zheng2019equipping} (see their remark after Lemma 14): it is indeed possible for FTRL with log-barrier to attain a $\sum_i x_{t,i}c_{t,i}^2$-style bound without any restrictions on the losses.
As log-barrier regularizers are widely used in the literature for its better data-adaptivity \citep{wei2018more,ito2021parameter}, we expect this result to be of independent interest. 
\end{remark}

\textbf{Linear-Q Algorithm.}
Our final algorithm for linear-Q MDPs is shown in \Cref{alg:linear-q with log-barrier},
which is nearly the same as \citep[Algorithm 2]{luo2021policy_nipsver} except for the part marked in blue where we use the log-barrier regularizer, as mentioned.
Specifically, based on the discussions in \Cref{sec:dialted_bonuses}, the loss fed to FTRL is 
$\hat Q_{k}-B_{k}$.
Here, $\hat Q_k$ (defined in \Cref{eq:Q_hat}) is a standard estimator for $Q_k^{\pi_k}$ involving an estimate $\hat \Sigma_{k,h}^\dagger$ for the inverse of $\Sigma_h^{\pi_k}$, which is constructed via the Matrix Geometric Resampling procedure~\citep{neu2020efficient} with the help of the simulator. Meanwhile, $B_k$ is the dilated bonus following the idea of \Cref{eq: dilated bonus} with $b_k(s,a)$ defined as follows for all $s\in \mS_h$ and $a\in \mA$:
\begin{equation*}
b_k(s,a) = \beta \left (\lVert \phi(s,a)\rVert_{\hat \Sigma_{k,h}^\dagger}^2+ \E_{\tilde a\sim \pi_k(\cdot \mid s)}\left [\lVert \phi(s,\tilde a)\rVert_{\hat \Sigma_{k,h}^\dagger}^2\right ] \right ).
\end{equation*}
The calculation of $B_k$ again requires the simulator -- see \Cref{alg:bonus calculation in linear Q} for the calculation procedure and \citep{luo2021policy_nipsver} for more detailed discussions.

\input{alg-linear-q-log-barrier}

With the help of the property of the log-barrier regularizer stated in \Cref{lem:log-barrier regret bound}, we are able to show that our algorithm achieves $\Otil(\sqrt{K})$ regret, formally stated as follows.
\begin{theorem}\label{thm:linear-q log-barrier main theorem}
\Cref{alg:linear-q with log-barrier} when applied to a linear-Q MDP (\Cref{def: linear Q}) with a simulator ensures the following when $12\eta \beta H^2\le \gamma$, $8\eta H^2\le \beta$, and $\epsilon\le (H^2K)^{-1}$:
\begin{align*}
\mathcal R_K=\Otil\bigg (&\beta d HK+\frac \gamma \beta dH^3 K+\frac{AH}{\eta}+\\&\frac \beta \gamma AH^2+A\sqrt dH^2+\frac{\eta H^3}{\gamma^2 K^2}\bigg ).
\end{align*}
Picking $\eta=\sqrt{\frac{A}{dH^4K}}$, $\beta=8\sqrt{\frac{A}{dK}}$, $\gamma=\frac{96A}{dK}$, and $\epsilon=\frac{1}{H^2K}$, we conclude that $\mathcal R_K=\Otil(H^3\sqrt{AdK})$.
\end{theorem}

We defer the full proof to \Cref{sec:appendix linear-q log-barrier} and provide a sketch below highlighting why removing any constraints on the losses fed to FTRL is critical to improving the regret.

\begin{proof}[Proof Sketch]
To bound $\mathcal R_K$ by the dilated bonus lemma (\Cref{lem:dilated bonus}), we focus on a single $(k,s)$-pair and bound
\begin{equation}\label{eq:objective of each state}
\textstyle \sum_a (\pi_k(a\mid s)-\pi^\ast(a\mid s))(Q_k^{\pi_k}(s,a)-B_k(s,a)).
\end{equation}
As the FTRL lemma only applies to the losses fed into \Cref{eq:OMD with log-barrier regularizer expression} (namely $\hat Q_k(s,a)-B_k(s,a)$), we add and substract $\hat Q_k(s,a)$ in \Cref{eq:objective of each state}.
Hence, after summing over $k\in [K]$ and $s\sim \pi^\ast$, we need to consider the following three terms to figure out the term $X(s)$ in \Cref{lem:dilated bonus}:
\begin{itemize}
\item {\small $\textsc{Bias-1}=\sum\limits_{k,h} \E\limits_{s_h\sim \pi^\ast}\bigg [\E\limits_{a_h\sim \pi_k}\bigg [Q_k^{\pi_k}(s_h,a_h)-\hat Q_k(s_h,a_h)\bigg ]\bigg ]$}, measuring the under-estimation of $\hat Q_k$ w.r.t. $\pi_k$.
\item {\small $\textsc{Bias-2}=\sum\limits_{k,h} \E\limits_{s_h\sim \pi^\ast}\bigg [\E\limits_{a_h\sim \pi^\ast}\bigg [\hat Q_k(s_h,a_h)-Q_k^{\pi_k}(s_h,a_h)\bigg ]\bigg ]$}, measuring the over-estimation of $\hat Q_k$ w.r.t. $\pi^\ast$.
\item {\small $\textsc{Reg-Term}=\sum\limits_{k,h}\E\limits_{s_h\sim \pi^\ast}\bigg [\sum\limits_a (\pi_k(a\vert s)-\pi^\ast(a\vert s))(\hat Q_k(s,a)-B_k(s,a))\bigg ]$}, which can be tackled by \Cref{lem:log-barrier regret bound}.
\end{itemize}

As \textsc{Bias-1} and \textsc{Bias-2} are independent of the regularizer, they are handled similarly to the original analysis and both contribute $\Otil(\frac \gamma \beta dH^3 K)$ 
to $\sum_{h}\E_{s_h\sim \pi^\ast} [X(s_h)]$ (see \Cref{lem:Bias-1 of log-barrier} in the appendix).
To handle $\textsc{Reg-Term}$, we apply \Cref{lem:log-barrier regret bound}.
As in standard log-barrier analyses, since $\Psi(\pi^\ast(\cdot \mid s))-\Psi(\pi_1(\cdot \mid s))$ is potentially infinity, we introduce a smooth version $\tilde \pi^\ast$ defined via
\begin{equation*}
\tilde \pi^\ast(a\mid s)=(1-AK^{-1}) \pi^\ast(a\mid s)+K^{-1}.
\end{equation*}

Applying \Cref{lem:log-barrier regret bound} to a given state $s\in \mS$, we derive
\begin{equation*}
\Psi(\tilde\pi^\ast(\cdot \mid s))-\Psi(\pi_1(\cdot \mid s)) \leq A\log K.
\end{equation*}

Hence, fixing the state $s$, we can write
\begin{align*}
&\scalemath{0.98}{\quad \sum_{k=1}^K\sum_{a\in \mA} (\pi_k(a\mid s)-\pi^\ast(a\mid s))(\hat Q_k(s,a)-B_k(s,a))} \\
&\scalemath{0.98}{\le \sum_{k=1}^K\sum_{a\in \mA} (\tilde \pi^\ast(a\mid s)-\pi^\ast(a\mid s))(\hat Q_k(s,a)-B_k(s,a))+}\\
&\scalemath{0.98}{\quad \frac{A\log K}{\eta} + \eta \sum_{k=1}^K \sum_{a\in \mA} \pi_k(a\mid s)^2 (\hat Q_k(s,a)-B_k(s,a))^2.}
\end{align*}
After summing over $h$ and taking expectation over $s$,
we show that the first term is of order $\Otil(AH^2(\sqrt{d}+\beta/\gamma))$, while the last term's contribution to $\sum_{h}\E_{s_h\sim \pi^\ast} [X(s_h)]$ is of order $\Otil(\tfrac{\eta H^3}{\gamma^2K^2})$ by the same analysis of~\citep{luo2021policy_nipsver}.
Finally, we combine everything, apply \Cref{lem:dilated bonus}, and show that the term $\sum_{k,h} \E_{(s_h,a_h)\sim \pi_k} [b_k(s_h,a_h)]$ in this case is $\Otil(\beta dHK)$, finishing the proof for the first bound.

While this bound looks almost identical to that of~\citep{luo2021policy_nipsver}, their analysis requires $2H\eta \leq \gamma$ because of the use of the negative-entropy regularizer.
This prevents them from picking a very small $\gamma$ and eventually leads to sub-optimal $\Otil(K^{2/3})$ regret.
On the other hand, our log-barrier analysis allows us to drop this requirement and pick a $\gamma$ as small as $K^{-1}$, though bearing some extra factors polynomial in $A$ (the number of actions).
Indeed, optimizing the bound over the parameters (see the second claim), we achieve $\mathcal R_K=\Otil(H^3\sqrt{AdK})$ which is of order $\sqrt K$.
\end{proof}

%% file: alg-linear-q-log-barrier.tex

\begin{algorithm}[t!]
\caption{Improved Linear-Q Algorithm\\(Using Log-Barrier Regularizers)}
\label{alg:linear-q with log-barrier}
\begin{algorithmic}[1]
\REQUIRE{Learning rate $\eta$, bonus parameter $\beta$, MGR parameters $\gamma$ and $\epsilon$, FTRL regularizer $\Psi(p)=\sum_{i=1}^{\lvert \mA\rvert} \ln \frac{1}{p_i}$.}
\FOR{$k=1,2,\ldots,K$}
\STATE Let $\pi_k\in \Pi$ be defined as follows for all $s
\in \mS$:
{\color{blue}
\begin{equation}\label{eq:OMD with log-barrier regularizer expression}
\scalemath{0.97}{
\begin{aligned}
\pi_k(s)=&\argmin_{p\in \triangle(\mA)}\bigg \{\Psi(p)+\\
&\quad \eta\sum_{k'<k} \langle p(\cdot),\hat Q_{k'}(s,\cdot)-B_{k'}(s,\cdot)\rangle\bigg \},
\end{aligned}}
\end{equation}
}
where $B_k(s,a)$ is calculated in \Cref{alg:bonus calculation in linear Q}.
\STATE Execute $\pi_k$, observing the trajectory $(s_{k,h},a_{k,h})$ and losses $\ell_k(s_{k,h},a_{k,h})$ for all $h\in[H]$.
\STATE Construct covariance matrix inverse estimate $\hat \Sigma_{k,h}^\dagger$ using Matrix Geometric Resampling (\Cref{alg:matrix geometric resampling} in the appendix) s.t. $\lVert \hat \Sigma_{k,h}^\dagger\rVert_2\le \frac 1\gamma$,
\begin{equation*}
\left \lVert \E[\hat \Sigma_{k,h}^\dagger]-(\gamma I+\Sigma^{\pi_k}_h)^{-1}\right \rVert_2 \le \epsilon,
\end{equation*}
and the following holds w.p. $1-K^{-3}$:
\begin{equation*}
\left \lVert (\hat \Sigma_{k,h}^\dagger)^{1/2}\Sigma^{\pi_k}_h(\hat \Sigma_{k,h}^\dagger)^{1/2}\right \rVert_2 \le 2.
\end{equation*}
\STATE Let $L_{k,h}=\sum_{h'\ge h}\ell_k(s_{k,h'},a_{k,h'})$.
Estimate the Q-function $Q_k^{\pi_k}(s,a)$ as follows (for $s \in \mS_h$):
\begin{align}
\hat Q_k(s,a)=\phi(s,a)^\trans \hat \Sigma_{k,h}^\dagger \phi(s_{k,h},a_{k,h}) L_{k,h}. \label{eq:Q_hat}
\end{align}
\ENDFOR
\end{algorithmic}
\end{algorithm}

%% file: linear-q-variance-reduced.tex

One drawback of using log-barrier is that the term $\Psi(y)-\Psi(x_1)$ from \Cref{lem:log-barrier regret bound} leads to $\text{poly}(A)$ dependency (while for negative entropy, this is only $\log A$).
%
%
To get around this issue, we go back to using negative entropy. Recall that the issue of this regularizer is that to ensure the same bound as \Cref{lem:log-barrier regret bound}, we require $\eta c_{t,i} \geq -1$, which translates to the requirement $\eta (\hat Q_k(s,a)-B_k(s,a)) \geq -1$ in the context of linear-Q MDPs.
The challenging part here is that $\hat Q_k(s,a)$, as defined in \Cref{eq:Q_hat}, could be as negative as $-H/\gamma$, which then restricts $\eta$ to be of order $\O(\gamma/H)$, as mentioned.

To resolve this, we propose a new Q-function estimator that has a smaller magnitude in the negative direction.
Our high-level idea is the following: for a possibly negative random variable $Z$, define another \textit{magnitude-reduced} random variable as (recall the notation that $(Z)_- = \min\{Z,0\}$):
\begin{equation*}
\hat Z = Z - (Z)_-+\E[(Z)_-].
\end{equation*}
The magnitude-reduced $\hat Z$ then has the same expectation as $Z$. They also share the same order of second moments:
\begin{equation*}
\E[\hat Z^2] \leq 2\E[Z^2] + 2(\E[(Z)_-])^2 = \O(\E[Z^2]).
\end{equation*}
More importantly, we have $\hat Z \geq \E[(Z)_-]$
and thus the smallest possible value of $\hat Z$ is $\E[(Z)_-]$: no less than (and often much larger than) the smallest possible value of $Z$. Thus, it becomes much easier to ensure $\eta c_{t,i} \geq -1$.

Applying this idea to our context, we propose a new Q-function estimator as in \Cref{eq:magnitude-reduced_Q_hat},
where $m_k(s,a)$ exactly takes the role of $\E[(Z)_-]$,
except that we have no access to the real expectation but have to approximate it with samples; see \Cref{line:m_k in variance-reduced approach} of \Cref{alg:linear-q using variance-reduced} for more details.

To make sure that these samples are ``consistent" with those used in constructing $\hat \Sigma_{k,h}^\dagger$,
we perform a check in \Cref{line:skipping criterion in variance-reduced approach} and repeat the sampling until the check passes.
Since both \Cref{eq:multiplicative error of Sigma in variance-reduced} and $\lVert \tilde \Sigma_{k,h}-\Sigma_{k,h}\rVert_2\le \gamma$ hold with probability at least $1-K^{-3}$, the expected number of trials is only $1+o(1)$.
Our algorithm is shown in \Cref{alg:linear-q using variance-reduced}, where the parts different from~\citep{luo2021policy_nipsver} are again highlighted in blue.

\input{alg-linear-q-variance-reduced}
Our analysis shows that the new Q-function estimator indeed has a significantly smaller magnitude: it can only be as negative as $-H/\sqrt{\gamma}$ (as opposed to the previous $-H/\gamma$ bound).
This only restricts $\eta$ to be of order $\O(\sqrt{\gamma}/H)$, enabling us to pick $\gamma \approx 1/K$ as in \Cref{thm:linear-q log-barrier main theorem} and achieve $\Otil(\sqrt{K})$ regret again, as formalized in the next theorem.
Note that our final regret bound not only has no $\text{poly}(A)$ dependency, but is also optimal in both $d$ and $K$ as one cannot do better even in the special case of linear bandits.



\begin{theorem}\label{thm:linear-q variance-reduced main theorem}
When applied to a linear-Q MDP (\Cref{def: linear Q}) with a simulator, \Cref{alg:linear-q using variance-reduced} ensures the following when $\eta \beta \gamma^{-1}\le \frac{1}{12H^2}$, $\eta^2 \gamma^{-1}\le \frac{1}{12H^2}$, $8\eta H^2\le \beta$, $\epsilon\le (H^2K)^{-1}$, and $M=32\gamma^{-2}\log K$:
\begin{align*}
\mathcal R_K=\Otil\bigg (&\beta d HK+\frac \gamma \beta dH^3 K+\frac H\eta+\\&\eta d H^3 K+\frac{\eta}{\gamma^2 K^2}H^3 \bigg ).
\end{align*}

Picking $\eta=\frac{1}{\sqrt{dH^4K}}$, $\beta=\frac{8}{\sqrt{dK}}$, $\gamma=\frac{96}{dK}$, and $\epsilon=\frac{1}{H^2K}$, we conclude that $\mathcal R_K=\Otil(H^3\sqrt{dK})$.
\end{theorem}


\begin{proof}[Proof Sketch]
Due to space limitations, we only sketch why $\hat Q_k(s,a)$ is now at least $-\O(H/\sqrt{\gamma})$:
since $L_{k,h} \in [0,H]$, we have $\hat Q_k(s,a) \geq Hm_k(s,a)$.
By Jensen's inequality, we can bound $m_k(s,a)^2$ for all $(s,a)$ as:
\begin{align*}
m_k(s,a)^2
&\le \frac 1M\sum_{m=1}^M\big (\phi(s,a)^\trans \hat \Sigma_{k,h}^\dagger \phi(s_{m,h},a_{m,h})\big )^2\\
&= \phi(s,a)^\trans \hat \Sigma_{k,h}^\dagger \tilde \Sigma_{h}^{\pi_k} \hat \Sigma_{k,h}^\dagger \phi(s,a)\\
&\leq 3\big \lVert \phi(s,a)\big \rVert_{\hat \Sigma_{k,h}^\dagger}^2 \leq \frac{3}{\gamma},
\end{align*}
where the second inequality is by \Cref{line:skipping criterion in variance-reduced approach}.
Therefore, we have $\hat Q_k(s,a) \geq -\frac{\sqrt{3}H}{\sqrt{\gamma}}$, a reduced magnitude compared to that of a standard estimator (which is defined in \Cref{eq:Q_hat}). 
\end{proof}

%% file: alg-linear-q-variance-reduced.tex

\begin{algorithm}[t!]
\caption{Improved Linear-Q Algorithm\\(Using Magnitude-Reduced Loss Estimators)}
\label{alg:linear-q using variance-reduced}
\begin{algorithmic}[1]
\REQUIRE{Learning rate $\eta$, bonus parameter $\beta$, MGR parameters $\gamma$ and $\epsilon$, covariance estimation parameter $M$.}
\FOR{$k=1,2,\ldots,K$}
\STATE Let $\pi_k\in \Pi$ be defined as follows for all $s\in \mS$:
\[
\scalemath{0.95}{
\begin{aligned}
\pi_k(a\mid s)&\propto \exp \bigg (-\eta \sum_{k'<k}(\hat Q_{k'}(s,a)-B_{k'}(s,a))\bigg ),
\end{aligned}}
\]
where $B_k(s,a)$ is calculated in \Cref{alg:bonus calculation in linear Q}.
\STATE Execute $\pi_k$, observing the trajectory $(s_{k,h},a_{k,h})$ and losses $\ell_k(s_{k,h},a_{k,h})$ for all $h\in[H]$.
\STATE Construct $\hat \Sigma_{k,h}^\dagger$ using Matrix Geometric Resampling (\Cref{alg:matrix geometric resampling} in the appendix) s.t. $\lVert \hat \Sigma_{k,h}^\dagger\rVert_2\le \frac 1\gamma$,
\begin{equation}\label{eq:error of Sigma in variance-reduced}
\left \lVert \E[\hat \Sigma_{k,h}^\dagger]-(\gamma I+\Sigma^{\pi_k}_h)^{-1}\right \rVert_2 \le \epsilon,
\end{equation}
and the following holds w.p. $1-K^{-3}$:
\begin{equation}\label{eq:multiplicative error of Sigma in variance-reduced}
\left \lVert (\hat \Sigma_{k,h}^\dagger)^{1/2}\Sigma^{\pi_k}_h(\hat \Sigma_{k,h}^\dagger)^{1/2}\right \rVert_2 \le 2.
\end{equation}
\STATE {\color{blue}Simulate $\pi_k$ for $M$ times, giving $\{(s_{m,h},a_{m,h})\}_{m,h}$. Estimate the covariance matrix $\Sigma_h^{\pi_k}$ as
\begin{equation*}
\tilde \Sigma_{k,h}=\frac 1M \sum_{m=1}^M \phi(s_{m,h},a_{m,h})\phi(s_{m,h},a_{m,h})^\trans.
\end{equation*}} \alglinelabel{line:samples_for_m_k}
\STATE {\color{blue} If $\lVert (\hat \Sigma_{k,h}^\dagger)^{1/2} \tilde \Sigma_{k,h}  (\hat\Sigma_{k,h}^\dagger)^{1/2}\rVert_2 \ge 3$, \textbf{goto} Line 4.}
\alglinelabel{line:skipping criterion in variance-reduced approach}
\STATE {\color{blue}For each $s\in \mS_h$ and $a\in \mA$, define $m_k(s,a)$ as
\begin{equation*}
\scalemath{0.95}{
m_k(s,a)=\frac 1M\sum_{m=1}^M \left (\phi(s,a)^\trans \hat \Sigma_{k,h}^\dagger \phi(s_{m,h},a_{m,h})\right )_-.}
\end{equation*}}
\alglinelabel{line:m_k in variance-reduced approach}
\STATE Let $L_{k,h}=\sum_{h'\ge h}\ell_k(s_{k,h'},a_{k,h'})$.
Estimate the Q-function $Q_k^{\pi_k}(s,a)$ as follows (for $s \in \mS_h$):
\begin{equation}\label{eq:magnitude-reduced_Q_hat}
\begin{split}
\hat Q_k(s,a)&=\phi(s,a)^\trans \hat \Sigma_{k,h}^\dagger \phi(s_{k,h},a_{k,h}) L_{k,h} \\
&{\color{blue} {}-H\left (\phi(s,a)^\trans \hat \Sigma_{k,h}^\dagger \phi(s_{k,h},a_{k,h})\right )_-}\\
&{\color{blue} {}+H m_k(s,a)}.
\end{split}
\end{equation}
\alglinelabel{line:hat Q_k in variance-reduced approach}
\ENDFOR
\end{algorithmic}
\end{algorithm}

%% file: linear-mdp.tex

When a simulator is unavailable, we consider a special case of linear-Q MDPs: linear MDPs (\Cref{def:linear MDP}), where the transition is also linear in the features.
As \citet{luo2021policy} show, this makes the dilated bonuses linear as well, and thus we can efficiently estimate them without using a simulator.
Due to various technical obstacles, they only achieve $\Otil(K^{14/15})$ regret.
We improve it to $\Otil(K^{8/9})$ via two key modifications to their algorithm.
As the algorithm is fairly lengthy due to the lack of simulators and the estimation of the dilated bonus function, we defer it (\Cref{alg:linear MDP without MGR}) to the appendix.
We refer the reader to \citep{luo2021policy} for a detail description of their algorithm, and only focus on our two modifications (highlighted in blue) described below.

The first obvious modification is to apply one of the techniques introduced in the last two sections.
However, since the magnitude-reduced estimator we developed in \Cref{alg:linear-q using variance-reduced} requires extra samples (see \Cref{line:samples_for_m_k}), which, without a simulator, can only be done via real episodes and in turn introduces more regret, we go with the log-barrier approach instead, even though it leads to extra $\text{poly}(A)$ factors.



But it turns out that this only leads to a mild improvement in the regret, since the constraint on $\eta$ is not the bottleneck of their analysis.
Instead, one important bottleneck comes from using the Matrix Geometric Resampling (MGR) procedure \citep{neu2020efficient} to estimate the covariance matrix inverse, which again, without a simulator, can only be done via running the same policy for multiple real episodes (called an \textit{epoch}) to collect samples.
It is thus critical to make this step as sample efficient as possible.

To resolve this issue, our key observation is that MGR uses too many samples to produce an estimate that is in a sense more accurate than required: specifically, it needs $\O(\epsilon^{-2}\gamma^{-3})$ samples to ensure a bound like \Cref{eq:error of Sigma in variance-reduced}.
Instead, we find that a weaker multiplicative approximation guarantee is enough, which only requires $\O(\gamma^{-2})$ samples.
Moreover, this is achieved by simply taking the (regularized) inverse of the empirical average of $\O(\gamma^{-2})$ samples.

More concretely, for a policy $\tilde \pi_j$ (which mixes $\pi_j$ with some exploration policy; see \Cref{alg:linear MDP without MGR} for more details) in epoch $j$ of the algorithm, we collect roughly $\O(\gamma^{-2})$ samples using this policy and construct an empirical average covariance matrix $\tilde \Sigma_{j,h}$ for layer $h$.
Then we let $\hat \Sigma_{j,h}^\dagger=(\gamma I+\tilde \Sigma_{j,h})^{-1}$ be the estimation for $(\gamma I+\Sigma_h^{\tilde \pi_j})^{-1}$ (see \Cref{line:covariance estimation in linear MDPs} of \Cref{alg:linear MDP without MGR}).
We then prove the following:
\begin{lemma}\label{lem:multiplicative error (no MGR)}
If the epoch length $W$ in \Cref{alg:linear MDP without MGR} is set to $(4d\log \frac d\delta) \gamma^{-2}$, then
\begin{equation*}
(1-\sqrt \gamma)(\gamma I+\Sigma_h^{\tilde \pi_j})\preceq (\gamma I+\tilde \Sigma_{j,h})\preceq (1+\sqrt \gamma)(\gamma I+\Sigma_h^{\tilde \pi_j})
\end{equation*}
holds with probability at least $1-2\delta$.
\end{lemma}
The proof relies on a new matrix concentration bound (\Cref{lem:new matrix concentration lemma} in the appendix), which can be of independent interest.
A direct corollary of this lemma is:
\begin{corollary}\label{corol:multiplicative error of inverse (no MGR)}
If $W=(4d\log \frac d\delta) \gamma^{-2}$ and $\gamma\le \frac 14$, then with probability $1-2\delta$, we have the following:
\begin{align*}
\scalemath{0.95}{
 (1-2\sqrt\gamma)I\preceq (\hat \Sigma_{j,h}^\dagger)^{1/2} (\gamma I+\Sigma_h^{\tilde \pi_j}) (\hat \Sigma_{j,h}^\dagger)^{1/2} \preceq (1+2\sqrt\gamma) I.}
\end{align*}
\end{corollary}
\begin{proof}
Simply left and right multiply $(\hat \Sigma_{j,h}^\dagger)^{1/2} = (\gamma I+\tilde \Sigma_{j,h})^{-1/2}$ in the bound of \Cref{lem:multiplicative error (no MGR)} and use the fact
 $[\frac{1}{1+\sqrt\gamma},\frac{1}{1-\sqrt \gamma}]\subseteq [1-2\sqrt\gamma,1+2\sqrt\gamma]$ when $\gamma\leq\frac 14$.
\end{proof}


Together with a new analysis to bound bias terms in the regret,
such approximations turn out to be enough: the bias terms enjoy almost the same bound (up to constants) as we had in previous sections. Hence, we finally get the following theorem; see \Cref{sec:appendix linear-mdp} for a full proof.
\begin{theorem}[Informal version of \Cref{thm:linear MDP without MGR main theorem formal}]\label{thm:linear MDP without MGR main theorem}
When applied to linear MDPs, \Cref{alg:linear MDP without MGR} with proper tuning ensures $\mathcal R_K=\Otil(K^{8/9})$ (omitting all other dependencies).
\end{theorem}
\begin{proof}[Proof Sketch]

Ignoring less important parts, at a high level we still decompose the regret into several bias terms and a \textsc{Reg-Term}. In this sketch, we only provide key ideas in bounding the bias terms using \Cref{corol:multiplicative error of inverse (no MGR)}.
Specifically, we illustrate how to bound $\E[\textsc{Bias-1}]$, defined as follows:
\begin{equation*}
\scalemath{0.9}{
\textsc{Bias-1}\triangleq W\sum\limits_{j,h} \E\limits_{s_h\sim \pi^\ast}\bigg [\E\limits_{a_h\sim \pi_j}\bigg [\bar Q_j^{\pi_j}(s_h,a_h)-\hat Q_j(s_h,a_h)\bigg ]\bigg ],}
\end{equation*}
where $\bar Q_j^{\pi_j}(s,a)=\phi(s,a)^\trans \bar \theta_{j,h}^{\pi_j}$ is the average Q-functions in epoch $j$ induced by policy $\pi_j$, 
and $\hat Q_j(s,a)$ is its estimator.
By direct calculation, one may check that
\begin{align*}
&\quad 
\E\left [\bar Q_j^{\pi_j}(s,a)-\hat Q_j(s,a)\right ]\\
&=\phi(s,a)^\trans (I-\hat \Sigma_{j,h}^\dagger \Sigma_h^{\tilde\pi_j}) \bar \theta_{j,h}^{\pi_j},\\
&=\phi(s,a)^\trans \hat \Sigma_{j,h}^\dagger (\gamma I+\tilde \Sigma_{j,h}-\Sigma_h^{\tilde\pi_j}) \bar \theta_{j,h}^{\pi_j}\\
&\le \left \lVert \phi(s,a)\right \rVert_{\hat \Sigma_{j,h}^\dagger}\times \left \lVert (\gamma I+\tilde \Sigma_{j,h}-\Sigma_h^{\tilde\pi_j}) \bar \theta_{j,h}^{\pi_j}\right \rVert_{\hat \Sigma_{j,h}^\dagger}\\
&\le \left \lVert \phi(s,a)\right \rVert_{\hat \Sigma_{j,h}^\dagger} \times \bigg(\left \lVert \gamma \bar \theta_{j,h}^{\pi_j}(s,a)\right \rVert_{\hat \Sigma_{j,h}^\dagger} + \\
&\quad \hspace{2.7cm} \left \lVert (\tilde \Sigma_{j,h}-\Sigma_h^{\tilde\pi_j}) \bar \theta_{j,h}^{\pi_j}\right \rVert_{\hat \Sigma_{j,h}^\dagger}\bigg)\\
&\le \frac \beta 4 \left \lVert \phi(s,a)\right \rVert_{\hat \Sigma_{j,h}^\dagger}^2 + \frac 2\beta \left \lVert \gamma \bar \theta_{j,h}^{\pi_j}(s,a)\right \rVert_{\hat \Sigma_{j,h}^\dagger}^2+\\
&\quad \hspace{0.12cm} \frac 2\beta \left \lVert (\tilde \Sigma_{j,h}-\Sigma_h^{\tilde\pi_j}) \bar\theta_{j,h}^{\pi_j}\right \rVert_{\hat \Sigma_{j,h}^\dagger}^2,
\end{align*}
where we apply Cauchy-Schwarz inequality, triangle inequality, and AM-GM inequality (twice).

The first term contributes to $\O(\beta dHK)$ after summing over $k$ and $h$ and applying \Cref{lem:dilated bonus}, while the second term is bounded by $\O(\frac \gamma \beta dH^2)$ for any $j$ because of the fact that $\lVert \hat \Sigma_{j,h}^\dagger\rVert_2\le \gamma^{-1}$.
By definition of $\bar \theta$, the last term becomes
\begin{align*}
\frac 2\beta \left \lVert (\tilde \Sigma_{j,h}-\Sigma_h^{\tilde\pi_j}) \hat \Sigma_{j,h}^\dagger (\tilde \Sigma_{j,h}-\Sigma_h^{\tilde\pi_j})\right \rVert_2 dH^2.
\end{align*}
By some algebraic manipulations, we write
\begin{align*}
&\quad (\tilde \Sigma_{j,h}-\Sigma_h^{\tilde\pi_j}) \hat \Sigma_{j,h}^\dagger (\tilde \Sigma_{j,h}-\Sigma_h^{\tilde\pi_j})\\
&=(\hat \Sigma_{j,h}^\dagger)^{-\nicefrac 12} \big (I-(\hat \Sigma_{j,h}^\dagger)^{\nicefrac 12} (\gamma I+\Sigma_{h}^{\tilde\pi_j}) (\hat \Sigma_{j,h}^\dagger)^{\nicefrac 12}\big )^2(\hat \Sigma_{j,h}^\dagger)^{-\nicefrac 12}.
\end{align*}

Thanks to \Cref{corol:multiplicative error of inverse (no MGR)}, we know that
\begin{equation*}
-2\sqrt\gamma I \preceq I-(\hat \Sigma_{j,h}^\dagger)^{1/2} (\gamma I+\Sigma_{h}^{\tilde\pi_j}) (\hat \Sigma_{j,h}^\dagger)^{1/2} \preceq 2\sqrt\gamma I.
\end{equation*}
Hence, the last term is also of order $\O(\frac \gamma \beta d H^2)$ -- same as the second term.
All other bias terms are bounded analogously.

It only remains to bound the \textsc{Reg-Term}, whose calculation is the same as we did in the proof of \Cref{thm:linear-q log-barrier main theorem}. Property tuning all the parameters then gives $\Otil(K^{8/9})$ regret.
\end{proof}

%% file: conclusion.tex

In this paper, we study policy optimization algorithms in adversarial MDPs with linear function approximation.

Building on top of the dilated bonus approach introduced by \citet{luo2021policy_nipsver}, we derive two algorithms which both achieve \textit{the first} $\Otil(\sqrt K)$-style regret bounds in linear-Q adversarial MDPs when a simulator is accessible. Technically speaking, the first algorithm uses a refined analysis for FTRL with the log-barrier regularizer, while the second one relies on a new magnitude-reduced loss estimator.

We further generalize the first approach to simulator-free linear MDPs and get $\Otil(K^{8/9})$ regret, greatly improving over the best-known $\Otil(T^{14/15})$ bound \citep{luo2021policy}. This generalization also contains an alternative to the Matrix Geometric Resampling procedure \citep{neu2020efficient} using a new matrix concentration bound (\Cref{lem:new matrix concentration lemma}).

We expect all these techniques to be of independent interest and potentially useful for other problems.
In light of various concurrent works on linear MDPs \citep{sherman2023improved,kong2023improved,lancewicki2023delay}, further improving the result for linear MDPs is a key future direction --- either in terms of $K$ or $A$.

%% file: appendix-algorithms.tex

\subsection{Matrix Geometric Resampling Procedure}
Matrix Geometric Resampling algorithm, introduced by \citet{neu2020efficient} and improved by \citet{luo2021policy_nipsver}, generates a close estimation of $(\gamma I+\Sigma_h^{\pi})^{-1}$ for given $\pi\in \Pi$ and $h\in [H]$. Formally, we present it in \Cref{alg:matrix geometric resampling} and state its performance guarantees in \Cref{lem:MGR lemma}.
\begin{algorithm}[htb]
\caption{Matrix Geometric Resampling Algorithm}
\label{alg:matrix geometric resampling}
\begin{algorithmic}[1]
\REQUIRE Policy $\pi$, regularization parameter $\gamma$, desired precision $\epsilon$.
\ENSURE A set of matrices $\{\hat \Sigma_h^\dagger\}_{h=1}^H$.
\STATE Set $M\ge \frac{24
\ln(dHT)}{\epsilon^2 \gamma^2}$ and $N\ge \frac 2\gamma \ln \frac{1}{\epsilon \gamma}$. Set $c=\frac 12$.
\STATE Draw $MN$ trajectories of $\pi$ using the simulator; denote them by $\{(s_{m,n,h},a_{m,n,h})\}_{m\in [M],n\in [N],h\in [H]}$.
\FOR{$m=1,2,\ldots,M$}
\FOR{$n=1,2,\ldots,N$}
\STATE Set $Y_{n,h}=\gamma I+\phi(s_{m,n,h},a_{m,n,h})\phi(s_{m,n,h},a_{m,n,h})^\trans$ and $Z_{n,h}=\prod_{n'=1}^n (I-c Y_{n',h})$ for all $h\in [H]$.
\ENDFOR
\STATE Calculate $\hat \Sigma_{m,h}^\dagger=cI+c\sum_{n=1}^N Z_{n,h}$ for all $h\in [H]$.
\ENDFOR
\STATE Set $\hat \Sigma_h^\dagger = \frac 1M\sum_{m=1}^M \hat \Sigma_{m,h}^\dagger$ for all $h\in [H]$.
\end{algorithmic}
\end{algorithm}

\begin{lemma}[Lemma D.1 of \citet{luo2021policy_nipsver}]\label{lem:MGR lemma}
With the configurations of $M$ and $N$ stated in \Cref{alg:matrix geometric resampling}, for any policy $\pi$, we have $\lVert \hat \Sigma_h^\dagger\rVert_2\le \gamma^{-1}$, $\lVert \E[\hat \Sigma_h^\dagger]-(\gamma I+\Sigma_h^\pi)^{-1}\rVert_2\le \epsilon$. Moreover, there exists a good event that happens with probability $1-\frac{1}{K^3}$, under which the following two properties hold:
\begin{equation*}
\lVert \E[\hat \Sigma_h^\dagger]-(\gamma I+\Sigma_h^\pi)^{-1}\rVert_2\le 2\epsilon,\quad \lVert \hat \Sigma_h^\dagger \Sigma_h^\pi \rVert_2\le 1+2\epsilon.
\end{equation*}
\end{lemma}

\subsection{Dilated Bonus Calculation in Linear-Q MDP Algorithms}
The following algorithm (which is Algorithm 3 of \citet{luo2021policy_nipsver}) indicates how we calculate the dilated bonus function $B_k(s,a)$ with the help of the simulator in \Cref{alg:linear-q with log-barrier,alg:linear-q using variance-reduced}. In both algorithms, we define the bonus as
\begin{equation*}
b_k(s,a)=\beta \bigg (\lVert \phi(s,a)\rVert_{\hat \Sigma_{j,h}^\dagger}^2+ \E_{\tilde a\sim \pi_k(\cdot \mid s)}\big [\lVert \phi(s,\tilde a)\rVert_{\hat \Sigma_{j,h}^\dagger}^2\big ] \bigg ).
\end{equation*}
\begin{algorithm}[htb]
\caption{Dilated Bonus Calculation in Linear-Q Algorithms}
\label{alg:bonus calculation in linear Q}
\begin{algorithmic}[1]
\REQUIRE Episode $k\in [K]$, state $s\in \mS$, action $a\in \mA$.
\ENSURE The dilated bonus function $B_k(s,a)$.
\STATE \textbf{if} $\textsc{Bonus}(k,s,a)$ is called before \textbf{then} \textbf{return} the value calculated at that time.
\STATE Let $h$ be the layer that $s$ lies in, i.e., $s\in \mS_h$. \textbf{if} $h=H+1$ \textbf{then} \textbf{return} 0.
\STATE Call the simulator for a next state $s'\sim \P(\cdot \mid s,a)$. Calculate $\pi_k(\cdot \mid s)$ and $\pi_k(\cdot \mid s')$ according to the algorithm.
\STATE \textbf{return} $\beta \bigg (\lVert \phi(s,a)\rVert_{\hat \Sigma_{k,h}^\dagger}^2+ \E_{\tilde a\sim \pi_k(\cdot \mid s)}\big [\lVert \phi(s,\tilde a)\rVert_{\hat \Sigma_{k,h}^\dagger}^2\big ] \bigg )+ (1+\frac 1H) B_k(s',a')$ where $a'$ is a sample from $\pi_k(\cdot \mid s')$.
\end{algorithmic}
\end{algorithm}

Note that, during the calculation of $\pi_k$ and $B_k(s',a')$, we are essentially recursively calling this algorithm.

\subsection{\textsc{PolicyCover} Algorithm in Linear MDP Algorithms}
The following algorithm is Algorithm 6 of \citet{luo2021policy} (which itself builds upon Algorithm 1 of \citet{wang2020reward}; we refer the readers to the original paper for more details).
This algorithm generates a mixture of policies, which we call the policy cover $\pi_{\text{cov}}$, together with an estimate of the (regularized) inverse of its covariance matrix, namely $\{\hat \Sigma_h^{\text{cov}}\}_{h=1}^H$.
It ensures the following property:
\begin{lemma}[Lemma D.4 by \citet{luo2021policy}]\label{lem:PolicyCover}
With probability $1-\delta$, we have the following for all policies $\pi$ and any $h\in [H]$:
\begin{equation*}
\Pr_{s_h\sim \pi}\left [s_h\not \in \mathcal K\right ]=\Otil\left (\frac{dH}{\alpha}\right ),
\end{equation*}
where $\mathcal K$ is set of known states defined as follows:
\begin{equation*}
\mathcal K=\left \{s
\in \mS\middle \vert \forall a\in \mA,\lVert \phi(s,a)\rVert_{(\hat \Sigma_h^{\text{cov}})^{-1}}^2\le \alpha\right \},
\end{equation*}
where $h$ is the layer that $s$ lies in.
\end{lemma}
\begin{algorithm}[htb]
\caption{\textsc{PolicyCover} for Linear MDPs}
\label{alg:PolicyCover}
\begin{algorithmic}[1]
\REQUIRE Configurations $M_0,N_0,\alpha$, failure probability $\delta$.
\STATE Set $\Gamma_{1,h}\gets I$ for all $h\in [H]$. Set $\tilde \beta=60dH\sqrt{\ln \frac K\delta}$.
\FOR{$m=1,2,\ldots,M_0$}
\STATE Set $\hat V_m(x_{h+1})\gets 0$.
\FOR{$h=H,H-1,\ldots,1$}
\STATE For all $(s,a)\in \mS_h\times \mA$, let
\begin{align*}
\hat Q_m(s,a)&=\min\{r_m(s,a)+\tilde \beta \lVert \phi(s,a)\rVert_{\Gamma_{m,h}^{-1}}+\phi(s,a)^\trans \hat \theta_{m,h},H\},\\
\hat V_m(s)&=\max_{a\in \mA}\hat Q_m(s,a),\\
\pi_m(a\mid s)&=\mathbbm 1[a=\argmax_{a'\in \mA}\hat Q_m(s,a')],
\end{align*}
where
\begin{align*}
r_m(s,a)&=\text{ramp}_{\frac 1K}(\lVert \phi(s,a)\rVert_{\Gamma_{m,h}^{-1}}-\frac{\alpha}{M_0}),\\
\hat \theta_{m,h}&=\Gamma_{m,h}^{-1}(\frac{1}{N_0}\sum_{m'=1}^{m-1}\sum_{n'=1}^{N_0}\phi(s_{m',n',h},a_{m',n',h})\hat V_m(s_{m',n',h+1})).
\end{align*}
Here, $\text{ramp}_z(y)$ is $0$ if $y\le -z$, $1$ if $y\ge 0$, and $\frac yz+1$ otherwise.
\ENDFOR
\FOR{$n=1,2,\ldots,N_0$}
\STATE Execute $\pi_m$ and get trajectory $\{(s_{m,n,h},a_{m,n,h})\}_{h=1}^H$.
\ENDFOR
\STATE Compute $\Gamma_{m+1,h}$ as
\begin{equation*}
\Gamma_{m+1,h}\gets \Gamma_{m,h}+\frac{1}{N_0}\sum_{n=1}^{N_0}\phi(s_{m,n,h},a_{m,n,h})\phi(s_{m,n,h},a_{m,n,h})^\trans.
\end{equation*}
\STATE \textbf{return} $\pi_{cov}$ defined as the uniform mixture of $\pi_1,\pi_2,\ldots,\pi_{M_0}$ and $\hat \Sigma_h^{\text{cov}}$ defined as $\frac{1}{M_0}\Gamma_{M_0+1,h}$ (for all $h$).
\ENDFOR
\end{algorithmic}
\end{algorithm}

\subsection{Stochastic Matrix Concentration}
We first state the well-known Matrix Azuma inequality.
\begin{lemma}[Matrix Azuma; see {\citep[Theorem 7.1]{tropp2012user}}]\label{lem:matrix azuma}
Let $\{X_k\}_{k=1}^n$ be a adapted sequence of $d\times d$ self-adjoint matrices. Let $\{A_k\}_{k=1}^n$ be a fixed sequence of self-adjoint matrices such that
\begin{equation*}
\E\nolimits_k[X_k]=0,\quad X_k^2\preceq A_k^2\text{ almost surely}.
\end{equation*}

Let $\sigma^2 =\lVert \frac 1n\sum_{k=1}^n A_k^2\rVert_2$. Then for all $\epsilon>0$:
\begin{equation*}
\Pr\left \{\left \lVert \frac 1n \sum_{k=1}^n X_k \right \rVert_2 \ge \epsilon\right \}\le d\exp\left (-\frac{n\epsilon^2}{8\sigma^2}\right ).
\end{equation*}
\end{lemma}

Then, we state the following lemma, which we use to replace the MGR procedure in Linear MDPs.
\begin{lemma}\label{lem:new matrix concentration lemma}
Let $H_1,H_2,\ldots,H_n$ be i.i.d. PSD matrices s.t. $\E[H_i]=H$, $H_i\preceq I$ a.s., and $H\succeq \frac{1}{dn}\log \frac d\delta I$, then with probability $1-2\delta$,
\begin{equation*}
-\sqrt{\frac dn\log \frac d\delta}H^{1/2}\preceq \frac 1n\sum_{i=1}^n H_i-H\preceq \sqrt{\frac dn\log \frac d\delta}H^{1/2}.
\end{equation*}
\end{lemma}
\begin{proof}
Let $G=\sqrt{\frac{1}{dn}\log \frac d\delta}H^{-1/2}$.
We first show that 
\begin{align*}
-\frac{2\log \frac{d}{\delta}}{n}G^{-1}\preceq\frac{1}{n}\sum_{i=1}^n H_i-H 
\end{align*}
holds with probability $1-\delta$. We have
\begin{align*}
\Pr\left\{-\frac{2\log(\frac{d}{\delta})}{n}G^{-1}\not\preceq\frac{1}{n}\sum_{i=1}^n H_i-H \right\}
=\Pr\left\{\sum_{i=1}^nG^{1/2}\left(H-H_i\right)G^{1/2}-\log(\frac{d}{\delta})I\not\preceq \log(\frac{d}{\delta})I\right\}.
\end{align*}
Since $\log$ is operator monotone, we have for PSD matrices $A,B$: $\exp(A)\preceq \exp(B) \,\Rightarrow\log(\exp(A))\preceq \log(\exp(B))$ (note that the reverse does not generally hold). We have
$\Pr\{\exp(A)\preceq \exp(B)\} \leq \Pr\{A\preceq B\}$ and hence $\Pr\{A\not\preceq B\}\leq \Pr\{\exp(A)\not\preceq \exp(B)\}$. Hence
\begin{align*}
    &\Pr\left\{\sum_{i=1}^nG^{1/2}\left(H-H_i\right)G^{1/2}-\log(\frac{d}{\delta})I\not\preceq \log(\frac{d}{\delta})I\right\}\\
    &\leq\Pr\left\{\exp\left(\sum_{i=1}^nG^{1/2}\left(H-H_i\right)G^{1/2}-\log(\frac{d}{\delta})I\right)\not\preceq \frac{d}{\delta}I\right\}\\
    &\leq\Pr\left\{\tr\left(\exp\left(\sum_{i=1}^nG^{1/2}\left(H-H_i\right)G^{1/2}-\log(\frac{d}{\delta})I\right)\right)>\frac{d}{\delta}\right\}\\
    &\leq \frac{\E\left[\tr\left(\exp\left(\sum_{i=1}^nG^{1/2}\left(H-H_i\right)G^{1/2}-\log(\frac{d}{\delta})I\right)\right)\right]}{d}\delta\tag{Markov's inequality}
\end{align*}
Using the Golden-Thompson inequality, we have
\begin{align*}
    &\E\left[\tr\left(\exp\left(\sum_{i=1}^nG^{1/2}\left(H-H_i\right)G^{1/2}-\log(\frac{d}{\delta})I\right)\right)\right]\\
    &\leq \E\left[ \tr\left(\exp\left(\sum_{i=1}^{n-1}G^{1/2}\left(H-H_i\right)G^{1/2}-\frac{\log(\frac{d}{\delta})}{n}I\right)\exp(G^{1/2}\left(H-H_n\right)G^{1/2}-\frac{\log(\frac{d}{\delta})}{n}I)\right)\right]\\
    &= \tr\left(\E\left[\exp\left(\sum_{i=1}^{n-1}G^{1/2}\left(H-H_i\right)G^{1/2}-\frac{\log(\frac{d}{\delta})}{n}I\right)\right]\E\left[\exp(G^{1/2}\left(H-H_n\right)G^{1/2}\right]\exp\left(-\frac{\log(\frac{d}{\delta})}{n}I)\right)\right)
\end{align*}
We have due to $G^{1/2}(H-H_n)G^{1/2}\preceq I$:\,
\begin{align*}
\E\left[\exp(G^{1/2}\left(H-H_n\right)G^{1/2})\right]
\preceq \E[I + (G^{1/2}\left(H-H_n\right)G^{1/2})+(G^{1/2}\left(H-H_n\right)G^{1/2})^2]
\preceq I+\E\left [(G^{1/2}H_nG^{1/2})^2\right ]\,.
\end{align*}
By the Araki–Lieb–Thirring inequality, we have $\tr((ABA)^2)\leq \tr(A^2B^2A^2)$, hence
\begin{align*}
\tr\left(\E[(G^{1/2}H_nG^{1/2})^2]\right)
&\leq \E[\tr(GH_n^2G)]\leq \E[\tr(GH_nG)]\tag{$H_n\preceq I$}\\
&=\tr(GHG) \leq \frac{\log(\frac{d}{\delta})}{n}.
\end{align*}
Combining this with the previous result yields
\begin{align*}
    \E\left[\exp(G^{1/2}\left(H-H_n\right)G^{1/2})\right]\exp\left(-\frac{\log(\frac{d}{\delta})}{n}I)\right)\preceq \left (1+\frac{\log(\frac{d}{\delta})}{n}\right )I\exp\left(-\frac{\log(\frac{d}{\delta})}{n}I\right)\preceq I\,.
\end{align*}
Applying this recursively yields
\begin{align*}
\E\left[\tr\left(\exp\left(\sum_{i=1}^nG^{1/2}\left(H-H_i\right)G^{1/2}-\log(\frac{d}{\delta})I\right)\right)\right]\leq \tr\left(I\right) =d \,,
\end{align*}
which concludes the proof of the claim. By symmetry and union bound, we have the following with probability $1-2\delta$:
\begin{equation*}
-\frac{2\log \frac{d}{\delta}}{n}G^{-1}\preceq\frac{1}{n}\sum_{i=1}^n H_i-H \preceq \frac{2\log \frac{d}{\delta}}{n}G^{-1}.
\end{equation*}

This then directly simplifies to our conclusion given $H\succeq \frac{1}{dn}\log \frac d\delta I$.
\end{proof}

%% file: appendix-linear-q-log-barrier.tex

\subsection{Property of FTRL with Log-Barrier Regularizers}\label{sec:log-barrier lemma proof}
\begin{proof}[Proof of \Cref{lem:log-barrier regret bound}]
We first introduce the notation of Bregman divergences $D_\Psi(y,x)=\Psi(y)-\Psi(x)-\langle \nabla \Psi(x),y-x\rangle$, which is heavily used in FTRL analyses. 
By standard FTRL analysis (see, e.g., Theorem 28.5 by \citet{lattimore2020bandit}), we know
\begin{align}
\sum_{t=1}^T \langle x_t-y,c_t\rangle\le \frac{\Psi(y)-\Psi(x_1)}{\eta}+\sum_{t=1}^T \left (\langle x_t-x_{t+1},c_t\rangle-\eta^{-1}D_\Psi(x_{t+1},x_t)\right ) .\label{eq:standard OMD analysis}
\end{align}

Consider $D_\Psi(x_{t+1},x_t)$. By definition of $D_\Psi$ and our choice that $\Psi(p)=\sum_{i=1}^A \ln \frac{1}{p_i}$, we have
\begin{align*}
D_\Psi(x_{t+1},x_t)&=\Psi(x_{t+1})-\Psi(x_t)-\langle \nabla \Psi(x_t),x_{t+1}-x_t\rangle\\
&=\sum_{i=1}^A \left (\ln \frac{x_t}{x_{t+1}}+\frac{x_{t+1}-x_t}{x_t}\right ).
\end{align*}

Observe that the following holds for all $x\in (0,1)$ and $x+\Delta\in (0,1)$:
\begin{equation*}
\ln \frac{x}{x+\Delta}+\frac \Delta x=\int_0^\Delta \frac{\alpha}{x(x+\alpha)}\mathrm{d}\alpha\ge \int_0^\Delta \frac{\alpha}{x}\mathrm{d}\alpha=\frac{\Delta^2}{2x},
\end{equation*}

For all $i$, we apply the inequality about with $x=x_{t,i}$ and $\Delta=x_{t+1,i}-x_{t,i}$  (as $x_t$ and $x_{t+1}$ both belong to the simplex, the conditions $x\in (0,1)$ and $(x+\Delta)\in (0,1)$ indeed hold). We then get the following lower bound on $D_\Psi(x_{t+1},x_t)$:
\begin{equation}\label{eq:key property of log-barrier}
D_\Psi(x_{t+1},x_t)\ge \sum_{i=1}^A \frac{(x_{t+1}-x_t)^2}{2x_t}.
\end{equation}

We further have
\begin{align*}
\langle x_t-x_{t+1},c_t\rangle-\eta^{-1} D_\Psi(x_{t+1},x_t)
\le \sum_{i=1}^A \left ((x_{t,i}-x_{t+1,i})c_{t,i}-\frac{(x_{t+1}-x_t)^2}{2x_t}\right )
\le \sum_{i=1}^A \frac 12 x_{t,i} \eta c_{t,i}^2,
\end{align*}
where the last step uses AM-GM inequality $-a^2+2ab\le b^2$. Plugging this back to \Cref{eq:standard OMD analysis} gives our conclusion.
\end{proof}

\subsection{Proof of Main Theorem} \label{sec:appendix linear-q log-barrier proof}
\begin{proof}[Proof of \Cref{thm:linear-q log-barrier main theorem}]
As sketched in the main text, we consider the following expression in order to apply the dilated bonus lemma (\Cref{lem:dilated bonus}):
\begin{align}
&\quad \sum_{k=1}^K\sum_{h=1}^H\E_{s_h\sim \pi^\ast}\left [\sum_{a_h\in \mA}(\pi_k(a_h\mid s_h)-\pi^\ast(a_h\mid s_h))(Q_k^{\pi_k}(s_h,a_h)-B_k(s_h,a_h))\right ] \nonumber\\
&=\underbrace{\sum_{k=1}^K\sum_{h=1}^H \E_{s_h\sim \pi^\ast} \left [\E_{a_h\sim \pi_k(\cdot \mid s_h)}\left [Q_k^{\pi_k}(s_h,a_h)-\hat Q_k(s_h,a_h)\right ]\right ]}_{\textsc{Bias-1}}+ \nonumber\\
&\quad \underbrace{\sum_{k=1}^K\sum_{h=1}^H \E_{s_h\sim \pi^\ast} \left [\E_{a_h\sim \pi^\ast(\cdot \mid s_h)}\left [\hat Q_k(s_h,a_h)-Q_k^{\pi_k}(s_h,a_h)\right ]\right ]}_{\textsc{Bias-2}}+ \nonumber\\
&\quad \underbrace{\sum_{k=1}^K\sum_{h=1}^H \E_{s_h\sim \pi^\ast} \left [\left \langle \pi_k(\cdot\mid s_h)-\pi^\ast(\cdot\mid s_h),\hat Q_k(s_h,\cdot)-B_k(s_h,\cdot) \right \rangle\right ]}_{\textsc{Reg-Term}}.\label{eq:regret decomposition}
\end{align}

According to \Cref{lem:Bias-1 of log-barrier}, we know that
\begin{align*}
\E[\textsc{Bias-1}]+\E[\textsc{Bias-2}]
&\le\frac \beta 4\E\left [\sum_{k=1}^K \sum_{h=1}^{H} \E_{s_h\sim \pi^\ast} \left [\E_{a_h\sim \pi_k(\cdot \mid s_h)}[\lVert \phi(s_h,a_h)\rVert_{\hat \Sigma_{k,h}^\dagger}^2] \right ]\right ]+\\
&\quad \frac \beta 4\E\left [\sum_{k=1}^K \sum_{h=1}^{H} \E_{s_h\sim \pi^\ast} \left [\E_{a_h\sim \pi^\ast(\cdot \mid s_h)}[\lVert \phi(s_h,a_h)\rVert_{\hat \Sigma_{k,h}^\dagger}^2] \right ]\right ]+ \O\left (\frac \gamma \beta dH^3 K+\epsilon(H+\beta)H K\right ).
\end{align*}

Moreover, by \Cref{lem:Reg-Term of log-barrier}, we can see that
\begin{align*}
\E[\textsc{Reg-Term}]
&\le H\eta^{-1}A\ln K+\O\left (AH^2\left (\epsilon+\sqrt d+\frac \beta \gamma\right )\right )+\\
&\quad 2\eta H^2 \sum_{k=1}^K \sum_{h=1}^H \E_{s_h\sim \pi^\ast}\left [\E_{a_h\sim \pi_k(\cdot \mid s_h)}\left [\lVert \phi(s_h,a_h)\rVert_{\hat \Sigma_{k,h}^\dagger}\right ]\right ]+\O\left (\frac{\eta H^3}{\gamma^2 K^2}\right )+\\
&\quad \frac 1H\sum_{k=1}^K \sum_{h=1}^H \E_{s_h\sim \pi^\ast}\left [\E_{a_h\sim \pi_k(\cdot \mid s_h)}\left [B_k(s_h,a_h)\right ]\right ].
\end{align*}

Plugging back into \Cref{eq:regret decomposition}, we know that
\begin{align*}
&\quad \sum_{k=1}^K\sum_{h=1}^H\E_{s_h\sim \pi^\ast}\left [\sum_{a_h\in \mA}(\pi_k(a_h\mid s_h)-\pi^\ast(a_h\mid s_h))(Q_k^{\pi_k}(s_h,a_h)-B_k(s_h,a_h))\right ]\\
&\le\frac \beta 4\E\left [\sum_{k=1}^K \sum_{h=1}^{H} \E_{s_h\sim \pi^\ast} \left [\E_{a_h\sim \pi_k(\cdot \mid s_h)}[\lVert \phi(s_h,a_h)\rVert_{\hat \Sigma_{k,h}^\dagger}^2] \right ]\right ]+\\
&\quad \frac \beta 4\E\left [\sum_{k=1}^K \sum_{h=1}^{H} \E_{s_h\sim \pi^\ast} \left [\E_{a_h\sim \pi^\ast(\cdot \mid s_h)}[\lVert \phi(s_h,a_h)\rVert_{\hat \Sigma_{k,h}^\dagger}^2] \right ]\right ]+\\
&\quad 2\eta H^2 \sum_{k=1}^K \sum_{h=1}^H \E_{s_h\sim \pi^\ast}\left [\E_{a_h\sim \pi_k(\cdot \mid s_h)}\left [\lVert \phi(s_h,a_h)\rVert_{\hat \Sigma_{k,h}^\dagger}^2\right ]\right ]+\\
&\quad \frac 1H\sum_{k=1}^K \sum_{h=1}^H \E_{s_h\sim \pi^\ast}\left [\E_{a_h\sim \pi_k(\cdot \mid s_h)}\left [B_k(s_h,a_h)\right ]\right ]+\\
&\quad \Otil \left (\frac \gamma \beta dH^3K+\epsilon (H+\beta)HK+\frac H\eta A+AH^2 \left (\epsilon +\sqrt d+\frac \beta \gamma\right )+\frac{\eta H^3}{\gamma^2K^2}\right ).
\end{align*}

Using $(6\eta H^2+\frac \beta 4)\le \beta$, we apply \Cref{lem:dilated bonus} with $b_k(s,a)=\lVert \phi(s,a)\rVert_{\hat \Sigma_{k,h}^\dagger}^2+\E_{a'\sim \pi_k(\cdot \mid s)}[\lVert \phi(s,a')\rVert_{\hat \Sigma_{k,h}^\dagger}^2]$, giving
\begin{align*}
\mathcal R_K&\le \beta \E\left [\sum_{k=1}^K \sum_{h=1}^{H} \E_{s_h\sim \pi_k} \left [\E_{a_h\sim \pi_k(\cdot \mid s_h)}[\lVert \phi(s_h,a_h)\rVert_{\hat \Sigma_{k,h}^\dagger}^2] \right ]\right ]+\\
&\quad \Otil \left (\frac \gamma \beta dH^3K+\epsilon (H+\beta)HK+\frac H\eta A+AH^2 \left (\epsilon +\sqrt d+\frac \beta \gamma\right )+\frac{\eta H^3}{\gamma^2K^2}\right ).
\end{align*}

Fixing an episode $k\in [K]$ and $h\in [H]$, we have
\begin{align}
&\quad \E_{s_h\sim \pi_k}\left [\sum_{a} \pi_k(a\mid s) \lVert \phi(s,a)\rVert_{\hat \Sigma_{k,h}^\dagger}^2\right ]\nonumber\\
&\le \E_{s_h\sim \pi_k}\left [\sum_{a} \pi_k(a\mid s) \lVert \phi(s,a)\rVert_{(\gamma I+\Sigma_h^{\pi_k})^{-1}}^2\right ]+\O(\epsilon)\nonumber\\
&\le\E_{s_h\sim \pi_k}\left [\sum_{a} \pi_k(a\mid s) \lVert \phi(s,a)\rVert_{(\Sigma_h^{\pi_k})^{-1}}^2\right ]+\O(\epsilon)\nonumber\\
&=\left \langle (\Sigma_h^{\pi_k})^{-1},\E_{(s_h,a_h)\sim \pi_k}[\phi(s_h,a_h) \phi(s_h,a_h)^\trans] \right \rangle=d.\label{eq:sum of bonus}
\end{align}

Therefore, we can conclude the following if $12\eta \beta H^2\le \gamma$ and $8\eta H^2\le \beta$:
\begin{equation*}
\mathcal R_K=\Otil \left (\beta dHK+\frac \gamma \beta dH^3K+\epsilon (H+\beta)HK+\frac H\eta A+AH^2 \left (\epsilon +\sqrt d+\frac \beta \gamma\right )+\frac{\eta H^3}{\gamma^2K^2}\right ).
\end{equation*}

It remains to tune the parameters. We first pick $\epsilon=\frac{1}{H^2K}$, which makes all terms related to $\epsilon$ constantly-bounded. Setting $\eta\approx K^{-1/2}$ and $\gamma\approx K^{-1}$, the last term is also $o(1)$. Hence, removing all constantly-bounded terms give
\begin{equation*}
\mathcal R_K=\Otil \left (AH\frac 1\eta+\beta dHK+\frac \gamma \beta dH^3K+AH^2 \frac \beta \gamma\right ).
\end{equation*}

We then get $\mathcal R_K=\Otil(\sqrt{AdH^6K})=\Otil(A^{1/2}d^{1/2}H^3K^{1/2})$ by picking:
\begin{equation*}
\eta=\sqrt{\frac{A}{dH^4K}},\beta=8\sqrt{\frac{A}{dK}},\gamma=96\frac{A}{dK}.
\end{equation*}

It's straightforward to verify that they satisfy $12\eta \beta H^2\le \gamma$ and $8\eta H^2\le \beta$.
\end{proof}

\subsection{Bounding \textsc{Bias-1} and \textsc{Bias-2}}
\begin{lemma}\label{lem:Bias-1 of log-barrier}
In \Cref{alg:linear-q with log-barrier}, we have
\begin{align*}
\E[\textsc{Bias-1}]+\E[\textsc{Bias-2}]
&\le\frac \beta 4\E\left [\sum_{k=1}^K \sum_{h=1}^{H} \E_{s_h\sim \pi^\ast} \left [\E_{a_h\sim \pi_k(\cdot \mid s_h)}[\lVert \phi(s_h,a_h)\rVert_{\hat \Sigma_{k,h}^\dagger}^2] \right ]\right ]+\\
&\quad \frac \beta 4\E\left [\sum_{k=1}^K \sum_{h=1}^{H} \E_{s_h\sim \pi^\ast} \left [\E_{a_h\sim \pi^\ast(\cdot \mid s_h)}[\lVert \phi(s_h,a_h)\rVert_{\hat \Sigma_{k,h}^\dagger}^2] \right ]\right ]+ \O\left (\frac \gamma \beta dH^3 K+\epsilon(H+\beta)H K\right ).
\end{align*}
\end{lemma}
\begin{proof}
As \textsc{Bias-1} has nothing to do with the choice of the regularizer, it can be bounded the same as the original algorithm \citep[Lemma D.2]{luo2021policy_nipsver}, which we also include below for completeness: fixing a specific $(k,s,a)$ and suppose that $s\in \mS_h$. Then we have the following, where every expectation is taken to the randomness in the $k$-th episode:
\begin{align*}
\E\left [Q_k^{\pi_k}(s,a)-\hat Q_k(s,a)\right ]
&=\phi(s,a)^\trans \theta_{k,h}^{\pi_k}-\phi(s,a)^\trans \E\left [\hat \Sigma_{k,h}^\dagger\phi(s_{k,h},a_{k,h})L_k\right ]\\
&=\phi(s,a)^\trans \theta_{k,h}^{\pi_k}-\phi(s,a)^\trans \E\left [\hat \Sigma_{k,h}^\dagger\right ] \E\left [\phi(s_{k,h},a_{k,h})\phi(s_{k,h},a_{k,h})^\trans \theta_{k,h}^{\pi_k}\right ]\\
&\overset{(a)}{\le }\phi(s,a)^\trans \left (I-(\gamma I+\Sigma_h^{\pi_k})^{-1}\Sigma_h^{\pi_k}\right )\theta_{k,h}^{\pi_k}+\epsilon H\\
&=\gamma \phi(s,a)^\trans (\gamma I+\Sigma_h^{\pi_k})^{-1}\theta_{k,h}^{\pi_k}+\epsilon H\\
&\overset{(b)}{\le} \gamma \lVert \phi(s,a)\rVert_{(\gamma I+\Sigma_h^{\pi_k})^{-1}}\lVert \theta_{k,h}^{\pi_k}\rVert_{(\gamma I+\Sigma_h^{\pi_k})^{-1}}+\epsilon H\\
&\overset{(c)}{\le}\frac \beta 4 \lVert \phi(s,a)\rVert_{(\gamma I+\Sigma_h^{\pi_k})^{-1}}^2+\frac{\gamma^2}{\beta}\lVert \theta_{k,h}^{\pi_k}\rVert_{(\gamma I+\Sigma_h^{\pi_k})^{-1}}^2+\epsilon H\\
&\overset{(d)}{\le} \frac \beta 4 \E\left [\lVert \phi(s,a)\rVert_{\hat \Sigma_{k,h}^\dagger}^2\right ] + \frac \gamma \beta dH^2+\epsilon H+\epsilon \frac \beta 4.
\end{align*}
where (a) used \Cref{eq:error of Sigma in variance-reduced} (which follows from \Cref{lem:MGR lemma}) and the assumption that $\lVert \phi(s,a)\rVert\le 1$ and $L_k\le H$, (b) used Cauchy-Schwartz inequality, (c) used AM-GM inequality, and (d) used the assumption that $\lVert \theta_{k,h}^{\pi_k}\rVert\le \sqrt dH$ (see \Cref{def: linear Q}) and again \Cref{eq:error of Sigma in variance-reduced}. Hence, we have
\begin{equation*}
\E[\textsc{Bias-1}]
\le\frac \beta 4\E\left [\sum_{k=1}^K \sum_{h=1}^{H} \E_{s_h\sim \pi^\ast} \left [\E_{a_h\sim \pi_k(\cdot \mid s_h)}[\lVert \phi(s_h,a_h)\rVert_{\hat \Sigma_{k,h}^\dagger}^2] \right ]\right ]+\O\left (\frac \gamma \beta dH^3 K+\epsilon(H+\beta)H K\right ).
\end{equation*}

Similarly, we have
\begin{equation*}
\E[\textsc{Bias-2}]
\le\frac \beta 4\E\left [\sum_{k=1}^K \sum_{h=1}^{H} \E_{s_h\sim \pi^\ast} \left [\E_{a_h\sim \pi^\ast(\cdot \mid s_h)}[\lVert \phi(s_h,a_h)\rVert_{\hat \Sigma_{k,h}^\dagger}^2] \right ]\right ]+\O\left (\frac \gamma \beta dH^3 K+\epsilon(H+\beta)H K\right ).
\end{equation*}

Combining these two parts together gives our conclusion.
\end{proof}

\subsection{Bounding \textsc{Reg-Term}}
\begin{lemma}\label{lem:Reg-Term of log-barrier}
Under the assumption that $12\eta \beta H^2\le \gamma$, we have the following in \Cref{alg:linear-q with log-barrier}:
\begin{align*}
\E[\textsc{Reg-Term}]
&\le H\eta^{-1}A\ln K+\O\left (AH^2\left (\epsilon+\sqrt d+\frac \beta \gamma\right )\right )+\\
&\quad 2\eta H^2 \sum_{k=1}^K \sum_{h=1}^H \E_{s_h\sim \pi^\ast}\left [\E_{a_h\sim \pi_k(\cdot \mid s_h)}\left [\lVert \phi(s_h,a_h)\rVert_{\hat \Sigma_{k,h}^\dagger}^2\right ]\right ]+\O\left (\frac{\eta H^3}{\gamma^2 K^2}\right )+\\
&\quad \frac 1H\sum_{k=1}^K \sum_{h=1}^H \E_{s_h\sim \pi^\ast}\left [\E_{a_h\sim \pi_k(\cdot \mid s_h)}\left [B_k(s_h,a_h)\right ]\right ].
\end{align*}
\end{lemma}
\begin{proof}
Using \Cref{lem:log-barrier regret bound}, we get the following for all $k\in [K]$, $s\in \mS_h$ (where $h\in [H]$), and $\tilde \pi^\ast\in \Pi$:
\begin{align*}
&\quad \sum_{k=1}^K\sum_{a\in \mA} (\pi_k(a\mid s)-\pi^\ast(a\mid s))(\hat Q_k(s,a)-B_k(s,a))\\
&\le \frac{\Psi(\tilde \pi^\ast(\cdot \mid s))-\Psi(\pi_1(\cdot \mid s))}{\eta}+\\
&\quad \sum_{k=1}^K\sum_{a\in \mA} (\tilde \pi^\ast(a\mid s)-\pi^\ast(a\mid s))(\hat Q_k(s,a)-B_k(s,a))+\\
&\quad \eta \sum_{k=1}^K \sum_{a\in \mA} \pi_k(a\mid s) (\hat Q_k(s,a)-B_k(s,a))^2.
\end{align*}

By picking $\tilde \pi^\ast(a\mid s)=(1-AK^{-1})\pi^\ast(a\mid s)+K^{-1}$, the first term is bounded by
\begin{align*}
\frac{\Psi(\tilde \pi^\ast(\cdot \mid s))-\Psi(\pi_1(\cdot \mid s))}{\eta}
&=\frac 1\eta\sum_{a\in \mA} \ln \frac{\pi_1(a\mid s)}{\tilde \pi^\ast(a\mid s)}
\le \frac 1\eta\sum_{a\in \mA} \ln \frac{\pi_1(a\mid s)}{K^{-1}}
\le \eta^{-1} A\ln K.
\end{align*}

Meanwhile, the second term is bounded by
\begin{align*}
&\quad \sum_{k=1}^K\sum_{a\in \mA} (\tilde \pi^\ast(a\mid s)-\pi^\ast(a\mid s))(\hat Q_k(s,a)-B_k(s,a))\\
&=\sum_{k=1}^K \sum_{a\in \mA} (-AK^{-1}\pi^\ast(a\mid s)+K^{-1})(\hat Q_k(s,a)-B_k(s,a)).
\end{align*}

Firstly, we have the following as $\lVert \hat \Sigma_{k,h}^\dagger\rVert_2\le \gamma^{-1}$:
\begin{align}\label{eq:magnitude of bonus in log-barrier}
B_k(s,a)&\le H\left (1+\frac 1H\right )^H \times 2\beta \sup_{s,a,h}\lVert \phi(s,a)\rVert_{\hat \Sigma_{k,h}^\dagger}^2\le 6 \beta H \gamma^{-1}.
\end{align}

Moreover, according to the calculation in \Cref{lem:Bias-1 of log-barrier}, we know that
\begin{equation*}
\E[\hat Q_k(s,a)-Q_k^{\pi_k}(s,a)]\le \gamma \phi(s,a)^\trans (\gamma I+\Sigma_h^{\pi_k})^{-1} \theta_{k,h}^{\pi_k}+\epsilon H\le \O((\epsilon+\sqrt d)H),
\end{equation*}
while, at the same time, $Q_k^{\pi_k}(s,a)\in [0,H]$.

Hence, after taking expectations on both sides, we know that
\begin{align*}
&\quad \E\left [\sum_{k=1}^K\sum_{a\in \mA} (\tilde \pi^\ast(a\mid s)-\pi^\ast(a\mid s))(\hat Q_k(s,a)-B_k(s,a))\right ]\\
&\le \sum_{k=1}^K \sum_{a\in \mA} (\lvert -AK^{-1}\pi^\ast(a\mid s)\rvert+\lvert K^{-1}\rvert)\lvert \E[\hat Q_k(s,a)-B_k(s,a)]\rvert\\
&\le K \times 2AK^{-1}\times \O\left ((\epsilon+\sqrt d)H+H+\frac \beta \gamma H\right )=\O\left (AH\left (\epsilon+\sqrt d+\frac \beta \gamma\right )\right ).
\end{align*}

Then consider the last term, which is directly bounded by
\begin{align*}
&\quad \eta \sum_{k=1}^K \sum_{a\in \mA} \pi_k(a\mid s)^2 (\hat Q_k(s,a)-B_k(s,a))^2\\
&\le 2\eta \sum_{k=1}^K \sum_{a\in \mA} \pi_k(a\mid s) \hat Q_k(s,a)^2+
2\eta \sum_{k=1}^K \sum_{a\in \mA} \pi_k(a\mid s) B_k(s,a)^2.
\end{align*}

The first term can be calculated as follows, following the original proof \citep[Lemma D.3]{luo2021policy_nipsver}:
\begin{align*}
\E[\hat Q_k(s,a)^2]
&\le H^2 \E[\phi(s,a)^\trans \hat \Sigma_{k,h}^\dagger \phi(s_{k,h},a_{k,h})\phi(s_{k,h},a_{k,h})^\trans \hat \Sigma_{k,h}^\dagger \phi(s,a)]\\
&=H^2 \E[\phi(s,a)^\trans \hat \Sigma_{k,h}^\dagger \Sigma_h^{\pi_k} \hat \Sigma_{k,h}^\dagger \phi(s,a)]\\
&\overset{(a)}{\le} 2H^2 \E[\phi(s,a)^\trans \hat \Sigma_{k,h}^\dagger \phi(s,a)]+\O\left (\frac{H^2}{\gamma^2 K^3}\right )=2H^2\E\left [\lVert \phi(s,a)\rVert_{\hat \Sigma_{k,h}^\dagger}^2\right ]+\O\left (\frac{H^2}{\gamma^2 K^3}\right ),
\end{align*}
where (a) used \Cref{eq:multiplicative error of Sigma in variance-reduced}, which happens with probability $1-K^{-3}$ for each $k$ (when it does not hold, we simply use the bound $\lVert \hat \Sigma_{k,h}^\dagger\rVert_2\le \gamma^{-1}$). Then we can conclude the following by adding back the summation over $h\in [H]$ and $s_h\sim \pi^\ast$:
\begin{align*}
\E[\textsc{Reg-Term}]
&\le H\eta^{-1}A\ln K+\O\left (AH^2\left (\epsilon+\sqrt d+\frac \beta \gamma\right )\right )+\\
&\quad 2\eta H^2 \sum_{k=1}^K \sum_{h=1}^H \E_{s_h\sim \pi^\ast}\left [\E_{a_h\sim \pi_k(\cdot \mid s_h)}\left [\lVert \phi(s_h,a_h)\rVert_{\hat \Sigma_{k,h}^\dagger}^2\right ]\right ]+\O\left (\frac{\eta H^3}{\gamma^2 K^2}\right )+\\
&\quad \frac 1H\sum_{k=1}^K \sum_{h=1}^H \E_{s_h\sim \pi^\ast}\left [\E_{a_h\sim \pi_k(\cdot \mid s_h)}\left [B_k(s_h,a_h)\right ]\right ],
\end{align*}
where the last term comes from the magnitude of $B_k$ (\Cref{eq:magnitude of bonus in log-barrier}) and the assumption that $12\eta \beta H^2\le \gamma$.
\end{proof}

%% file: appendix-linear-q-variance-reduced.tex

\subsection{Property of FTRL with Negative-Entropy Regularizers (a.k.a. Hedge)}
The following result is a classic result for the Hedge algorithm. For the sake of completeness, we also include a proof here.
\begin{lemma}\label{lem:hedge lemma}
Let $x_0,x_1,x_2,\ldots,x_T\in \mathbb R^A$ be defined as
\begin{equation*}
x_{t+1,i}=\left .\left (x_{t,i}\exp(-\eta c_{t,i})\right )\middle /\left (\sum_{i'=1}^A x_{t,i'}\exp(-\eta c_{t,i'})\right )\right .,\quad \forall 0\le t<T,
\end{equation*}
where $c_t\in \mathbb R^A$ is the loss corresponding to the $t$-th iteration. Suppose that $\eta c_{t,i}\ge -1$ for all $t\in [T]$ and $i\in [A]$. Then
\begin{equation*}
\sum_{t=1}^T \langle x_t-y,c_t\rangle\le \frac{\log A}{\eta}+\eta \sum_{t=1}^T \sum_{i=1}^A x_{t,i}c_{t,i}^2
\end{equation*}
holds for any distribution $y\in \triangle([A])$ when $x_0=(\frac 1A,\frac 1A,\ldots,\frac 1A)$.
\end{lemma}
\begin{proof}
By linearity, it suffices to prove the inequality for all one-hot $y$'s. Without loss of generality, let $y=\bm{1}_{i^\ast}$ where $i^\ast\in [A]$. Define $C_{t,i}=\sum_{t'=1}^t c_{t',i}$ as the prefix sum of $c_{t,i}$. Let
\begin{equation*}
    \Phi_t=\frac 1\eta \ln \left (\sum_{i=1}^A \exp\left (-\eta C_{t,i}\right )\right ),
\end{equation*}

then by definition of $x_t$, we have
\begin{align*}
    \Phi_t-\Phi_{t-1}&=\frac 1\eta \ln \left (\frac{\sum_{i=1}^A \exp(-\eta C_{t,i})}{\sum_{i=1}^A \exp(-\eta C_{t-1,i})}\right )
    =\frac 1\eta \ln \left (\sum_{i=1}^A x_{t,i} \exp(-\eta c_{t,i})\right )\\
    &\overset{(a)}{\le} \frac 1\eta \ln \left (\sum_{i=1}^A x_{t,i} (1-\eta c_{t,i}+\eta^2 c_{t,i}^2)\right )
    =\frac 1\eta \ln \left (1-\eta \langle x_t,c_t\rangle+\eta^2 \sum_{i=1}^A x_{t,i}c_{t,i}^2\right )\\
    &\overset{(b)}{\le}-\langle x_t,c_t\rangle+\eta \sum_{i=1}^A x_{t,i}c_{t,i}^2,
\end{align*}

where (a) used $\exp(-x)\le 1-x+x^2$ for all $x\ge -1$ and (b) used $\ln(1+x)\le x$ (again for all $x\ge -1$). Therefore, summing over $t=1,2,\ldots,T$ gives
\begin{align*}
    \sum_{t=1}^T \langle x_t,c_t\rangle&\le \Phi_0-\Phi_T+\eta \sum_{t=1}^T \sum_{i=1}^A x_{t,i}c_{t,i}^2\\
    &\le \frac{\ln N}{\eta}-\frac 1\eta \ln \left (\exp(-\eta C_{T,i^\ast})\right )+\eta \sum_{t=1}^T \sum_{i=1}^N p_t(i)\ell_t^2(i)\\
    &\le \frac{\ln A}{\eta}+L_T(i^\ast)+\eta \sum_{t=1}^T \sum_{i=1}^A x_{t,i}c_{t,i}^2.
\end{align*}

Moving $C_{t,i^\ast}$ to the LHS then shows the inequality for $y=\bm{1}_{i^\ast}$. The result then extends to all $y\in \triangle([A])$ by linearity.
\end{proof}

\subsection{Proof of Main Theorem}
\begin{proof}[Proof of \Cref{thm:linear-q variance-reduced main theorem}]
We first consider \Cref{line:skipping criterion in variance-reduced approach} in the algorithm. As sketched in the main text, we shall expect such an operation to be repeated for $1+o(1)$ times because \Cref{eq:multiplicative error of Sigma in variance-reduced} and $\lVert \tilde \Sigma_{k,h}-\Sigma_{h}^{\pi_k}\rVert_2\le \gamma$ both happens with probability $1-K^{-3}$ --- the first claim follows from \Cref{lem:MGR lemma} and the second one comes from \Cref{lem:matrix azuma} (where we set $X_m=\phi(s_{m,h},a_{m,h})\phi(s_{m,h},a_{m,h})^\trans-\Sigma_h^{\pi_k}$ and $A_{m,h}=I\succeq X_{m,h}$). With these two conditions and the fact that $\lVert \hat \Sigma_{k,h}^\dagger\rVert_2\le \gamma^{-1}$, the desired condition trivially holds.
Therefore, such an operation brings neither extra regret nor extra computational complexity, and we focus on the regret analysis from now on.

We define \textsc{Bias-1}, \textsc{Bias-2}, and \textsc{Reg-Term} exactly the same as \Cref{thm:linear-q log-barrier main theorem} (i.e., \Cref{eq:regret decomposition}). The \textsc{Bias-1} and \textsc{Bias-2} terms are bounded by \Cref{lem:Bias-1 of variance-reduced}, as follows:
\begin{align*}
\E[\textsc{Bias-1}]+\E[\textsc{Bias-2}]
&\le\frac \beta 4\E\left [\sum_{k=1}^K \sum_{h=1}^{H} \E_{s_h\sim \pi^\ast} \left [\E_{a_h\sim \pi_k(\cdot \mid s_h)}[\lVert \phi(s_h,a_h)\rVert_{\hat \Sigma_{k,h}^\dagger}^2] \right ]\right ]+\\
&\quad \frac \beta 4\E\left [\sum_{k=1}^K \sum_{h=1}^{H} \E_{s_h\sim \pi^\ast} \left [\E_{a_h\sim \pi^\ast(\cdot \mid s_h)}[\lVert \phi(s_h,a_h)\rVert_{\hat \Sigma_{k,h}^\dagger}^2] \right ]\right ]+ \O\left (\frac \gamma \beta dH^3 K+\epsilon(H+\beta) HK\right ).
\end{align*}

Assuming $12\eta^2H^2\le \gamma$ and $12\eta \beta H^2\le \gamma$, \textsc{Reg-Term} is bounded by \Cref{lem:Reg-Term of variance-reduced} as
\begin{align*}
\E[\textsc{Reg-Term}]\le H\eta^{-1}\ln A&+
6\eta H^2 \sum_{k=1}^K \sum_{h=1}^{H} \E_{s_h\sim \pi^\ast} \left [\E_{a_h\sim \pi_k(\cdot \mid s_h)}\left [\lVert \phi(s_h,a_h)\rVert_{\hat \Sigma_{k,h}^\dagger}^2\right ]\right ]\\
&+ \frac 1H \sum_{k=1}^K \sum_{h=1}^{H} \E_{s_h\sim \pi^\ast} \left [\E_{a_h\sim \pi_k(\cdot \mid s_h)}\left [B_k(s_h,a_h)\right ]\right ].
\end{align*}

Plugging into the regret decomposition, we get
\begin{align*}
&\quad \sum_{k=1}^K\sum_{h=1}^H\E_{s_h\sim \pi^\ast}\left [\sum_{a_h\in \mA}(\pi_k(a_h\mid s_h)-\pi^\ast(a_h\mid s_h))(Q_k^{\pi_k}(s_h,a_h)-B_k(s_h,a_h))\right ]\\
&\le \frac \beta 4\E\left [\sum_{k=1}^K \sum_{h=1}^{H} \E_{s_h\sim \pi^\ast} \left [\E_{a_h\sim \pi_k(\cdot \mid s_h)}[\lVert \phi(s_h,a_h)\rVert_{\hat \Sigma_{k,h}^\dagger}^2] \right ]\right ]+\\
&\quad \frac \beta 4\E\left [\sum_{k=1}^K \sum_{h=1}^{H} \E_{s_h\sim \pi^\ast} \left [\E_{a_h\sim \pi^\ast(\cdot \mid s_h)}[\lVert \phi(s_h,a_h)\rVert_{\hat \Sigma_{k,h}^\dagger}^2] \right ]\right ]+\\
&\quad 6\eta H^2 \E\left [\sum_{k=1}^K \sum_{h=1}^{H} \E_{s_h\sim \pi^\ast} \left [\E_{a_h\sim \pi_k(\cdot \mid s_h)}\left [\lVert \phi(s_h,a_h)\rVert_{\hat \Sigma_{k,h}^\dagger}^2\right ]\right ]\right ]+\\
&\quad \frac 1H \E\left [\sum_{k=1}^K \sum_{h=1}^{H} \E_{s_h\sim \pi^\ast} \left [\E_{a_h\sim \pi_k(\cdot \mid s_h)}\left [B_k(s_h,a_h)\right ]\right ]\right ]+\Otil\left (\frac H\eta+\frac \gamma \beta dH^3 K+\epsilon(H+\beta) HK\right ).
\end{align*}

Using the condition that $(6\eta H^2+\frac \beta 4)\le \beta$, we can apply \Cref{lem:dilated bonus} to conclude that
\begin{equation*}
\mathcal R_K=\Otil\left (\E\left [\beta \sum_{h=1}^H \E_{s_h\sim \pi_k}\left [\sum_{k=1}^K \sum_{a} \pi_k(a\mid s) \lVert \phi(s,a)\rVert_{\hat \Sigma_{k,h}^\dagger}^2\right ]\right ]+\frac H\eta+\frac \gamma \beta dH^3 K+\epsilon(H+\beta) HK\right ).
\end{equation*}

By \Cref{eq:sum of bonus}, we can conclude the following when assuming $8\eta H^2\le \beta$:
\begin{equation*}
\mathcal R_K=\Otil\left (\frac H\eta + \frac \gamma \beta dH^3 K+\epsilon(H+\beta)HK+\beta dHK\right ).
\end{equation*}

Plugging in the configurations that (again, one can see that this configuration satisfies all the conditions)
\begin{equation*}
\eta=\frac{1}{\sqrt{dK}H^2},\beta=\frac{8}{\sqrt{dK}}, \gamma=\frac{96}{dK},\epsilon=\frac{1}{H^2K},
\end{equation*}
we then conclude that $\mathcal R_K=\Otil(\sqrt{dH^6K})$.
\end{proof}

\subsection{Bounding \textsc{Bias-1} and \textsc{Bias-2}}
\begin{lemma}\label{lem:Bias-1 of variance-reduced}
In \Cref{alg:linear-q using variance-reduced}, we have
\begin{align*}
\E[\textsc{Bias-1}]+\E[\textsc{Bias-2}]
&\le\frac \beta 4\E\left [\sum_{k=1}^K \sum_{h=1}^{H} \E_{s_h\sim \pi^\ast} \left [\E_{a_h\sim \pi_k(\cdot \mid s_h)}[\lVert \phi(s_h,a_h)\rVert_{\hat \Sigma_{k,h}^\dagger}^2] \right ]\right ]+\\
&\quad \frac \beta 4\E\left [\sum_{k=1}^K \sum_{h=1}^{H} \E_{s_h\sim \pi^\ast} \left [\E_{a_h\sim \pi^\ast(\cdot \mid s_h)}[\lVert \phi(s_h,a_h)\rVert_{\hat \Sigma_{k,h}^\dagger}^2] \right ]\right ]+ \O\left (\frac \gamma \beta dH^3 K+\epsilon(H+\beta) HK\right ).
\end{align*}
\end{lemma}
\begin{proof}
Fixing a specific $(k,s,a)$ and assume $s\in \mS_h$. Then we have (again, all expectations are taken w.r.t. randomness in the $k$-th episode)
\begin{align*}
\E\left [\hat Q_k(s,a)\right ]
&=\E\left [\phi(s,a)^\trans \hat \Sigma_{k,h}^\dagger \phi(s_{k,h},a_{k,h})L_{k,h}\right ]-H\E\left [\left (\phi(s,a)^\trans \hat \Sigma_{k,h}^\dagger\phi(s_{k,h},a_{k,h})\right )_-\right ]+\\
&\quad H\E\left [\frac 1M\sum_{m=1}^M\left (\phi(s,a)^\trans \hat \Sigma_{k,h}^\dagger\phi(s_{m,h},a_{m,h})\right )_-\right ]\\
&=\E\left [\phi(s,a)^\trans \hat \Sigma_{k,h}^\dagger\phi(s_{k,h},a_{k,h})L_{k,h}\right ],
\end{align*}

where the last equality is because $(s_{k,h},a_{k,h})$ and $(s_{m,h},a_{m,h})$ are both sampled from $\pi_k$. The rest of the proof is then identical to \Cref{lem:Bias-1 of log-barrier}.
\end{proof}

\subsection{Bounding \textsc{Reg-Term}}
\begin{lemma}\label{lem:Reg-Term of variance-reduced}
Suppose that $12 H^2 \eta^2\le \gamma$, $12\eta \beta H^2\le \gamma$, $8\eta H^2\le \gamma$, and $M=32\gamma^{-2}\log K$. Then in \Cref{alg:linear-q using variance-reduced}, we have
\begin{align*}
\E[\textsc{Reg-Term}]\le H\eta^{-1}\ln A&+
6\eta H^2 \E\left [\sum_{k=1}^K \sum_{h=1}^{H} \E_{s_h\sim \pi^\ast} \left [\E_{a_h\sim \pi_k(\cdot \mid s_h)}\left [\lVert \phi(s_h,a_h)\rVert_{\hat \Sigma_{k,h}^\dagger}^2\right ]\right ]\right ]\\
&+ \frac 1H \E\left [\sum_{k=1}^K \sum_{h=1}^{H} \E_{s_h\sim \pi^\ast} \left [\E_{a_h\sim \pi_k(\cdot \mid s_h)}\left [B_k(s_h,a_h)\right ]\right ]\right ].
\end{align*}
\end{lemma}
\begin{proof}
To apply the Hedge lemma (\Cref{lem:hedge lemma}), we need to ensure that
\begin{equation*}
\eta (\hat Q_k(s,a)-B_k(s,a))\ge -1,\quad \forall (s,a)\in \mS\times \mA,k\in [K].
\end{equation*}

Fix an $(s,a,k)$ tuple and assume that $s\in \mS_h$. We have
\begin{align*}
\hat Q_k(s,a)-Hm_k(s,a)
&=\phi(s,a)^\trans \hat \Sigma_{k,h}^\dagger \phi(s_{k,h},a_{k,h})L_{k,h}-H (\phi(s,a)^\trans \hat \Sigma_{k,h}^\dagger \phi(s_{k,h},a_{k,h}))_-\\
&=\begin{cases}
\phi(s,a)^\trans \hat \Sigma_{k,h}^\dagger \phi(s_{k,h},a_{k,h})L_{k,h},&\phi(s,a)^\trans \hat \Sigma_{k,h}^\dagger \phi(s_{k,h},a_{k,h})\ge 0\\
-\phi(s,a)^\trans \hat \Sigma_{k,h}^\dagger \phi(s_{k,h},a_{k,h})(H-L_{k,h}),&\phi(s,a)^\trans \hat \Sigma_{k,h}^\dagger \phi(s_{k,h},a_{k,h})<0
\end{cases}\\
&\ge 0,
\end{align*}
as $L_{k,h}\in [0,H]$. Hence, $\hat Q_k(s,a)\ge H m_k(s,a)$ always holds. We consider $\hat Q_k(s,a)\ge H m_k(s,a)$ and $B_k(s,a)$ separately.

We first claim that $\eta H m_k(s,a)\ge -\frac 12$, which ensures $\eta \hat Q_k(s,a)\ge -\frac 12$. To see this, we only need to show $\eta^2 H^2 m_k(s,a)^2\le \frac 14$. By definition, $m_k(s,a)^2$ can be written as the following:
\begin{align*}
m_k(s,a)^2&=\left (\frac 1M\sum_{m=1}^M\left (\phi(s,a)^\trans \hat \Sigma_{k,h}^\dagger \phi(s_{m,h},a_{m,h})\right )_-\right )^2\\
&\overset{(a)}{\le} \frac 1M\sum_{m=1}^M\left (\phi(s,a)^\trans \hat \Sigma_{k,h}^\dagger \phi(s_{m,h},a_{m,h})\right )_-^2\\
&= \phi(s,a)^\trans \hat \Sigma_{k,h}^\dagger \left (\frac 1M\sum_{m=1}^M \phi(s_{m,h},a_{m,h})\phi(s_{m,h},a_{m,h})^\trans \right ) \hat \Sigma_{k,h}^\dagger \phi(s,a)\\
&= \phi(s,a)^\trans \hat \Sigma_{k,h}^\dagger \tilde \Sigma_{k,h} \hat \Sigma_{k,h}^\dagger \phi(s,a)\\
&\overset{(b)}{\le} 3 \phi(s,a)^\trans \hat \Sigma_{k,h}^\dagger \phi(s,a)\le 3\gamma^{-1}.
\end{align*}
where (a) used Jensen inequality, (b) used \Cref{line:skipping criterion in variance-reduced approach} of \Cref{alg:linear-q using variance-reduced}, and the last step uses the fact that $\lVert \hat \Sigma_{k,h}^\dagger\rVert_2\le \gamma^{-1}$.
Hence, it only remains to ensure that $12 H^2 \eta^2 \gamma^{-1}\le 1$, which is guaranteed by the assumption.



Meanwhile, we claim that $\eta B_k(s,a)\le \frac 12$. By definition of $B_k(s,a)$ and the fact that $\lVert \hat \Sigma_{k,h}^\dagger\rVert_2\le \gamma^{-1}$, we have
\begin{align}\label{eq:upper bound on B_k in linear-q using variance-reduced}
\eta B_k(s,a)&\le \eta H\left (1+\frac 1H\right )^H \times 2\beta \sup_{s,a,h}\lVert \phi(s,a)\rVert_{\hat \Sigma_{k,h}^\dagger}^2\le 6\eta \beta H \gamma^{-1},
\end{align}
which is bounded by $\frac{1}{2H}$ according to the condition that $12\eta \beta H^2\le \gamma$.

Therefore, fixing $h\in [H]$ and $s\in \mS_h$, we can apply the Hedge lemma (\Cref{lem:hedge lemma}):
\begin{align*}
&\quad \E\left [\sum_{k=1}^K \sum_a (\pi_k(a\mid s)-\pi^\ast(a\mid s))(\hat Q_k(s,a)-B_k(s,a))\right ]\\
&\le \frac{\ln A}{\eta}+
2\eta \E\left [\sum_{k=1}^K \sum_{a} \pi_k(a\mid s)\hat Q_k(s,a)^2\right ]+
2\eta \E\left [\sum_{k=1}^K \sum_{a} \pi_k(a\mid s) B_k(s,a)^2\right ].
\end{align*}

For the second term, we can write
\begin{align*}
\hat Q_k(s,a)^2&=(\phi(s,a)^\trans \hat \Sigma_{k,h}^\dagger \phi(s_{k,h},a_{k,h}) L_{k,h}-H(\phi(s,a)^\trans \hat \Sigma_{k,h}^\dagger \phi(s_{k,h},a_{k,h}))_-+Hm_k(s,a))^2\\
&\le 2H^2(\phi(s,a)^\trans \hat \Sigma_{k,h}^\dagger \phi(s_{k,h},a_{k,h}))^2+2H^2(\phi(s,a)^\trans \hat \Sigma_{k,h}^\dagger \phi(s_{k,h},a_{k,h}))_-^2+2H^2m_k^2(s,a).
\end{align*}

After taking expectations on both sides, we have
\begin{align*}
\E[\hat Q_k^2(s,a)]&\le 2H^2 \E[(\phi(s,a)^\trans \hat \Sigma_{k,h}^\dagger \phi(s_{k,h},a_{k,h}))^2]+2H^2 \E[(\phi(s,a)^\trans \hat \Sigma_{k,h}^\dagger \phi(s_{k,h},a_{k,h}))_-^2]+\\
&\quad 2 H^2\E\left [\frac 1M\sum_{m=1}^M(\phi(s,a)^\trans \hat \Sigma_{k,h}^\dagger \phi(s_{m,h},a_{m,h}))_-\right ]^2\\&\le 6H^2 \E[(\phi(s,a)^\trans \hat \Sigma_{k,h}^\dagger \phi(s_{k,h},a_{k,h}))^2],
\end{align*}
where we used $(X)_-^2\le X^2$, Jensen's inequality, and the fact that $(s_{k,h},a_{k ,h})$ and $(s_{m,h},a_{m,h})$ are both sampled from $\pi_k$. Meanwhile, we can also calculate that
\begin{align*}
&\quad \E[(\phi(s,a)^\trans \hat \Sigma_{k,h}^\dagger \phi(s_{k,h},a_{k,h}))^2]\\
&= \E \left [\phi(s,a)^\trans \hat \Sigma_{k,h}^\dagger \phi(s_{k,h},a_{k,h}) \phi(s_{k,h},a_{k,h})^\trans \hat \Sigma_{k,h}^\dagger \phi(s,a)\right ]\\
&=\E \left [\phi(s,a)^\trans \hat \Sigma_{k,h}^\dagger \Sigma_k^{\pi_k} \hat \Sigma_{k,h}^\dagger \phi(s,a)\right ]
\le \E\left [\lVert \phi(s,a)\rVert_{\hat \Sigma_{k,h}^\dagger}^2\right ],
\end{align*}
where the last inequality again uses \Cref{line:skipping criterion in variance-reduced approach} of \Cref{alg:linear-q using variance-reduced}. Hence, after summing up over the expectations when $s_h\sim \pi^\ast$ for all $h\in[H]$, we can conclude that
\begin{align*}
\E[\textsc{Reg-Term}]\le H\eta^{-1}\ln A&+
6\eta H^2 \sum_{k=1}^K \sum_{h=1}^{H} \E_{s_h\sim \pi^\ast} \left [\E_{a_h\sim \pi_k(\cdot \mid s_h)}\left [\lVert \phi(s_h,a_h)\rVert_{\hat \Sigma_{k,h}^\dagger}^2\right ]\right ]\\
&+ \frac 1H \sum_{k=1}^K \sum_{h=1}^{H} \E_{s_h\sim \pi^\ast} \left [\E_{a_h\sim \pi_k(\cdot \mid s_h)}\left [B_k(s_h,a_h)\right ]\right ],
\end{align*}
where the last term comes from \Cref{eq:upper bound on B_k in linear-q using variance-reduced} and the condition that $12\eta \beta H^2\le \gamma$.
\end{proof}

%% file: alg-linear-mdp.tex

\begin{algorithm}[!]
\caption{Improved Linear MDP Algorithm}
\label{alg:linear MDP without MGR}
\begin{algorithmic}[1]
\REQUIRE{Learning rate $\eta$, \textsc{PolicyCover} parameters $\alpha$ and $T_0$, bonus parameter $\beta$, covariance estimation parameter $\gamma\in (0,\frac 14)$, epoch length $W$, FTRL regularizer $\Psi(p)=\sum_{i=1}^{A} \ln \frac{1}{p_i}$, exploration probability $\delta_e$}
\STATE Let $\pi_{\text{cov}}$ and $\hat \Sigma_h^{\text{cov}}$ be the outputs of the \textsc{PolicyCover} algorithm (see \Cref{alg:PolicyCover}).
\STATE Let $\mathcal K=\{s\in \mS\mid \lVert \phi(s,a)\rVert_{(\hat \Sigma_h^{\text{cov}})^{-1}}^2\le \alpha,\forall a\in \mA\}$ be all the ``known'' states (defined in \Cref{lem:PolicyCover}).
\FOR{$j=1,2,\ldots,J=(T-T_0)/W$}
\STATE Calculate $\pi_j\in \Pi$ as follows for all $s\in \mS_h$:
{\color{blue}\begin{align*}
\pi_{j}(s) 
=\argmin_{p\in \triangle(\mA)}\bigg \{&\Psi(p)+\eta \sum_{\tau < j}\sum_{a\in \mA} p(a)\big (\hat Q_{\tau}(s,a)-\hat B_{\tau}(s,a)\big )\bigg \},
\end{align*}}
where $\hat Q_\tau(s,a)=\phi(s,a)^\trans \hat \theta_{\tau,h}$, $\hat B_\tau(s,a)=b_\tau(s,a)+\phi(s,a)^\trans \hat \Lambda_{\tau,h}$, and the bonus function is defined as
\begin{equation*}
b_\tau(s,a)=\mathbbm{1}[s\in \mathcal K]\times \beta \left (\lVert \phi(s,a)\rVert_{\hat \Sigma_{\tau,h}^\dagger}^2+\E_{a'\sim \pi_\tau(\cdot \mid s)}\left [\lVert \phi(s,a')\rVert_{\hat \Sigma_{\tau,h}^\dagger}^2\right ]\right ).
\end{equation*}
\STATE Randomly partition the episodes in the current epoch, i.e., $\{K_0+(j-1)W+1,\ldots,K_0+jW\}$, into two halves $\mT_j$ and $\mT_j'$, such that where $\lvert \mT_j\rvert=\lvert \mT_j'\rvert=\frac W2$.
\FOR{$k=K_0+(j-1)W+1,\ldots,K_0+jW$}
\STATE Let $Y_k$ be a sample from a Bernoulli distribution $\text{Ber}(\delta_e)$.
\IF{$Y_k=0$}
\STATE Execute $\pi_j$ for this episode and observe $\{(s_{k,h},a_{k,h})\}_{h=1}^H$ (together with $\ell_k(s_{k,h},a_{k,h})$).
\ELSIF{$Y_k=1$ and $k\in \mT_j$}
\STATE Execute $\pi_{\text{cov}}$  and observe $\{(s_{k,h},a_{k,h})\}_{h=1}^H$ together with $\ell_k(s_{k,h},a_{k,h})$.
\ELSE
\STATE Let $h_k$ be uniformly sampled from $[H]$. Execute $\pi_{\text{cov}}$ for steps $1,2,\ldots,h-1$ and $\pi_j$ for the remaining ones. Again, observe the trajectory $\{(s_{k,h},a_{k,h})\}_{h=1}^H$ and the losses $\ell_k(s_{k,h},a_{k,h})$.
\ENDIF
\ENDFOR
\STATE Define $\tilde \pi_j$ as the mixture of $(1-\delta_e)$ times $\pi_j$ and $\delta_e$ times $\pi_{\text{cov}}$ (i.e., expected policy played in epoch $j$).
\STATE {\color{blue}Estimate the covariance matrix $\Sigma_h^{\tilde \pi_j}$ as follows, and estimate $(\gamma I+\Sigma_h^{\tilde \pi_j})^{-1}$ by $\hat \Sigma_{j,h}^\dagger=(\gamma I+\tilde \Sigma_{j,h})^{-1}$.
\begin{equation*}
\tilde \Sigma_{j,h}=\frac{1}{\lvert \mT_j\rvert}\sum_{k\in \mT_j} \phi(s_{k,h},a_{k,h})\phi(s_{k,h},a_{k,h})^\trans,
\end{equation*}} \alglinelabel{line:covariance estimation in linear MDPs}
\STATE Estimate the average Q-function kernel $\bar \theta_{j,h}^{\pi_j}\triangleq \frac{1}{\lvert \mT_j'\rvert}\sum_{k\in \mT_j'} \theta_{k,h}^{\pi_j}$ as follows:
\begin{equation*}
\hat \theta_{j,h}=\hat \Sigma_{j,h}^\dagger \left (\frac{1}{\lvert \mT_j'\rvert}\sum_{k\in \mT_j'} \left ((1-Y_k)+Y_kH\mathbbm 1[h=h_k]\right ) \phi(s_{k,h},a_{k,h})L_{k,h}\right ),
\end{equation*}
where $L_{k,h}=\sum_{h'=h}^H \ell_{k}(s_{k,h'},a_{k,h'})$.
\alglinelabel{line:average kernel definition}
\STATE Estimate the average dilated bonus kernel $\bar \Lambda_{j,h}^{\pi_j}\triangleq \frac{1}{\lvert \mT_j'\rvert}\sum_{k\in \mT_j'} \Lambda_{k,h}^{\pi_j}$ as follows:
\begin{equation*}
\hat \Lambda_{j,h}=\hat \Sigma_{j,h}^\dagger \left (\frac{1}{\lvert \mT_j'\rvert}\sum_{k\in \mT_j'} \left ((1-Y_k)+Y_kH\mathbbm 1[h=h_k]\right )\phi(s_{k,h},a_{k,h})D_{k,h}\right ),
\end{equation*}
where $D_{k,h}=\sum_{h'=h+1}^{H}(1+\frac 1H)^{i-h} b_j(s_{k,h},a_{k,h})$.\alglinelabel{line:dilated bonus kernel estimation}
\ENDFOR
\end{algorithmic}
\end{algorithm}

%% file: appendix-linear-mdp.tex

\subsection{Pseudocode of the Improved Linear MDP Algorithm}

This section briefly discusses the linear MDP algorithm (presented in \Cref{alg:linear MDP without MGR}). Apart from the new covariance estimation technique introduced in the main text, it is also different from \Cref{alg:linear-q with log-barrier} in some other aspects due to the distinct nature of linear-Q MDPs and (simulator-free) linear MDPs, listed as follows:

Firstly, as there are no simulators, we cannot calculate the dilated bonus function recursively like \Cref{alg:bonus calculation in linear Q}. Fortunately, in linear MDPs, as observed by \citet{luo2021policy}, for any $k$ associated with some policy $\pi_k$ and bonus function $b_k$, we can write $B_k(s,a)$ defined in \Cref{eq: dilated bonus} as $B_k(s,a)=b_k(s,a)+\phi(s,a)^\trans \Lambda_{k,h}^{\pi_k}$, where ($\nu$ is defined in \Cref{def:linear MDP})
\begin{equation*}
\Lambda_{k,h}^\pi\triangleq \left (1+\frac 1H\right )\int_{s_{h+1}}\E_{a_{h+1}\sim \pi(\cdot \mid s_{h+1})}\left [B_k(s_{h+1},a_{h+1})\right ]\nu(s_{h+1})\mathrm{d}s_{h+1}.
\end{equation*}
Hence, $B_k(s,a)$ is also linear in $\phi(s,a)$. This allows us to estimate the $B_k$ just like $Q_k^{\pi_k}$, as we see in \Cref{line:dilated bonus kernel estimation}.

Notice that, as there are no more simulators, we cannot directly ``assume'' a good covariance estimation like in \Cref{alg:linear-q with log-barrier} (which ensures \Cref{eq:error of Sigma in variance-reduced,eq:multiplicative error of Sigma in variance-reduced}). Instead, we should divide the time horizon into several \textit{epochs} and execute (nearly) the same policy during each epoch to ensure a good estimation. See \Cref{alg:linear MDP without MGR} for more details.

\subsection{Alternative to the Matrix Geometric Resampling Procedure}
\begin{proof}[Proof of \Cref{lem:multiplicative error (no MGR)}]
By definition, we know the following holds for all $k\in \mT_j$:
\begin{equation*}
\E[\phi(s_{k,h},a_{k,h})\phi(s_{k,h},a_{k,h})^\trans]=\Sigma_h^{\tilde \pi_j},\quad \phi(s_{k,h},a_{k,h})\phi(s_{k,h},a_{k,h})^\trans\preceq I
\end{equation*}

Moreover, each $(s_{k,h},a_{k,h})$ is i.i.d. Thus, we can apply \Cref{lem:new matrix concentration lemma} with
\begin{equation*}
H_i=\frac 12\left (\gamma I+\phi(s_{k_i,h},a_{k_i,h})\phi(s_{k_i,h},a_{k_i,h})^\trans\right ),
H=\frac 12\left (\gamma I+\Sigma_h^{\tilde \pi_j}\right ),
\quad i=1,2,\ldots,\lvert \mT_j\rvert,
\end{equation*}
where $k_i$ stands for the $i$-th element in $\mT_j$.
We then have the following according to \Cref{lem:new matrix concentration lemma}:
\begin{align}\label{eq:result of the new concentration}
-\sqrt{\frac{d}{\lvert \mT_j\rvert} \log \frac d\delta} H^{1/2}
\preceq \frac{1}{\lvert \mT_j\rvert} \sum_{i=1}^{\lvert \mT_j\rvert} H_i-H
\preceq \sqrt{\frac{d}{\lvert \mT_j\rvert} \log \frac d\delta} H^{1/2},\quad \text{if } \gamma\ge 2\frac{d}{\lvert \mT_j\rvert}\log \frac d\delta.
\end{align}

Let the empirical average of all $H_i$'s be $\hat H$, i.e.,
\begin{equation*}
\hat H=\frac{1}{\lvert \mT_j\rvert}\sum_{i=1}^{\lvert \mT_j\rvert} H_i=\frac 12(\gamma I+\tilde \Sigma_{j,h}).
\end{equation*}

Then we can arrive at the following under the same condition as \Cref{eq:result of the new concentration}:
\begin{equation*}
I-\sqrt{\frac{d}{\lvert \mT_j\rvert}\log \frac d\delta} H^{-1/2}\preceq \hat HH^{-1}\preceq I+\sqrt{\frac{d}{\lvert \mT_j\rvert}\log \frac d\delta} H^{-1/2}.
\end{equation*}

Moreover, by the definition of $H$, we know $H^{-1/2}\preceq \sqrt 2(\gamma I)^{-1/2}$. Setting $W=4d\log \frac d\delta \gamma^{-2}$ (which ensures $\gamma\ge 2\frac{d}{\lvert \mT_j\rvert}\log \frac d\delta$ as $\lvert \mT_j\rvert=\frac W2$), the LHS and RHS become $(1-\sqrt \gamma) I$ and $(1+\sqrt \gamma) I$, respectively. Hence,
\begin{equation*}
(1-\sqrt \gamma)\tfrac 12(\gamma I+\Sigma_h^{\tilde \pi_j})
\preceq \tfrac 12 (\gamma I+\tilde \Sigma_{j,h})
\preceq (1+\sqrt \gamma)\tfrac 12(\gamma I+\Sigma_h^{\tilde \pi_j}),
\end{equation*}
which gives our conclusion after multiplying $2$ on both sides.
\end{proof}

\subsection{Proof of Main Theorem}
We first state the formal version of \Cref{thm:linear MDP without MGR main theorem}:
\begin{theorem}\label{thm:linear MDP without MGR main theorem formal}
Suppose that $\alpha=\frac{\delta_e}{6\beta}$, $M_0\ge \alpha^2 dH^2$, $N_0\ge 100\frac{M_0^3}{\alpha^2}\log \frac K\delta$, $W=4d\log \frac d\delta \gamma^{-2}$, $\gamma\ge 36\frac{\beta^2}{\delta_e^2}$, and $100\eta H^4\le \beta$. Further pick $\delta=K^{-3}$. Then \Cref{alg:linear MDP without MGR} applied to linear MDPs (\Cref{def:linear MDP}) ensures
\begin{equation*}
\mathcal R_K=\Otil\left (\frac{\delta_e^6}{\beta^6} d^4 H^8+\delta_e K+\beta dHK+\frac \gamma \beta dH^3 K+\frac H \eta Ad\gamma^{-2}+Ad^{3/2}H^2\gamma^{-2}+\frac{H^3}{\gamma^2} K\right ).
\end{equation*}

With some proper tuning, we can ensure $\mathcal R_K=\Otil((H^{20}Ad^6)^{1/9}K^{8/9})$.
\end{theorem}
\begin{proof}[Proof of \Cref{thm:linear MDP without MGR main theorem}]
The regret decomposition is the same as Theorem 6.1 by \citet{luo2021policy}, which we include below.
As sketched in the main text, we decompose the episodes into three parts: those executing \textsc{PolicyCover}, those using exploratory policies (i.e., the $\text{Ber}(\delta_e)$ gives 1), and the ones executing $\pi_j$.

For the first part, it's trivially bounded by $\O(K_0H)$. For the second part, as we explore with probability $\delta_e$ for each episode, the total regret gets bounded by $\O(\delta_e KH)$. For the last part, it suffices to bound the following to apply \Cref{lem:dilated bonus} (where we still consider those exploratory episodes as they only bring extra regret), where $J=\frac{K-K_0}{W}$ denotes the number of epochs for simplicity and $\bar Q_j^\pi(s,a)=\phi(s,a)^\trans \bar \theta_{j,h}^\pi$ denotes the average Q-function in $\mT_j'$ ($h$ is such that $s\in \mS_h$):
\begin{equation*}
\sum_{j=1}^{J}\sum_{h=1}^H \E_{s_h\sim \pi^\ast}\left [W\sum_{a_h\in \mA}(\pi_j(a_h\mid s_h)-\pi^\ast(a_h\mid s_h))(\bar Q_j^{\pi_j}(s_h,a_h)-B_j(s_h,a_h))\right ].
\end{equation*}

As $B_k(s,a)$ is only bounded for the known states (by definition of $\mathcal K$), we first consider the unknown states:
\begin{align*}
&\quad \sum_{j=1}^J\sum_{h=1}^H \E_{s_h\sim \pi^\ast}\left [W\mathbbm{1}[x\not \in \mathcal K]\sum_{a_h\in \mA} (\pi_j(a_h\mid s_h)-\pi^\ast(a_h\mid s_h))(\bar Q_j^{\pi_j}(s,a)-B_j(s_h,a_h))\right ]\\
&\le J\sum_{h=1}^E \E_{s_h\sim \pi^\ast}[W\mathbbm 1[s_h\not \in \mathcal K]3H]=\Otil\left (HK\frac{dH^3}{\alpha}\right)
\end{align*}
according to \Cref{lem:PolicyCover}. For the remaining, we decompose $(\pi_j(a_h\mid s_h)-\pi^\ast(a_h\mid s_h))(\bar Q_j^{\pi_j}(s_h,a_h)-B_j(s_h,a_h))$ into the biases of $\hat Q_j(s_h,a_h)$ and $\hat B_j(s,a_h)$ plus the FTRL regret $(\pi_j(a_h\mid s_h)-\pi^\ast(a_h\mid s_h))(\hat Q_j(s_h,a_h)-\hat B_j(s_h,a_h))$, i.e.,
\begin{align*}
&\quad \sum_{j=1}^{J}\sum_{h=1}^H \E_{s_h\sim \pi^\ast}\left [W\mathbbm 1[s_h\in \mathcal K]\sum_{a_h\in \mA}(\pi_j(a_h\mid s_h)-\pi^\ast(a_h\mid s_h))(\bar Q_j^{\pi_j}(s_h,a_h)-B_j(s_h,a_h))\right ]\\
&=W\underbrace{\sum_{j=1}^{J}\sum_{h=1}^H \E_{s_h\sim \pi^\ast}\left [\mathbbm 1[s_h\in \mathcal K]\E_{a_h\sim \pi_j(\cdot \mid s_h)}\left [\bar Q_j^{\pi_j}(s_h,a_h)-\hat Q_j(s_h,a_h))\right ]\right ]}_{\textsc{Bias-1}}+\\
&\quad W\underbrace{\sum_{j=1}^{J}\sum_{h=1}^H \E_{s_h\sim \pi^\ast}\left [\mathbbm 1[s_h\in \mathcal K]\E_{a_h\sim \pi^\ast(\cdot \mid s_h)}\left [\hat Q_j(s_h,a_h))-\bar Q_j^{\pi_j}(s_h,a_h)\right ]\right ]}_{\textsc{Bias-2}}+\\
&\quad W\underbrace{\sum_{j=1}^{J}\sum_{h=1}^H \E_{s_h\sim \pi^\ast}\left [\mathbbm 1[s_h\in \mathcal K]\E_{a_h\sim \pi_j(\cdot \mid s_h)}\left [\hat B_j(s_h,a_h)-B_j(s_h,a_h)\right ]\right ]}_{\textsc{Bias-3}}+\\
&\quad W\underbrace{\sum_{j=1}^{J}\sum_{h=1}^H \E_{s_h\sim \pi^\ast}\left [\mathbbm 1[s_h\in \mathcal K]\E_{a_h\sim \pi^\ast(\cdot \mid s_h)}\left [B_j(s_h,a_h)-\hat B_j(s_h,a_h)\right ]\right ]}_{\textsc{Bias-4}}+\\
&\quad W\underbrace{\sum_{j=1}^{J}\sum_{h=1}^H \E_{s_h\sim \pi^\ast}\left [\mathbbm 1[s_h\in \mathcal K]\sum_{a_h\in \mA}(\pi_j(a_h\mid s_h)-\pi^\ast(a_h\mid s_h))(\hat Q_j(s_h,a_h)-\hat B_j(s_h,a_h))\right ]}_{\textsc{Reg-Term}}.
\end{align*}

All the bias terms can be bounded similarly, as we will show in \Cref{lem:bias-1 in linear MDPs,lem:bias-3 in linear MDPs}, we can bound them as
\begin{align*}
&\quad \E[\textsc{Bias-1}]+\E[\textsc{Bias-2}]+\E[\textsc{Bias-3}]+\E[\textsc{Bias-4}]\\
&\le\frac \beta 2\E\left [\sum_{j=1}^J \sum_{h=1}^{H} \E_{s_h\sim \pi^\ast} \left [\E_{a_h\sim \pi_j(\cdot \mid s_h)}[\lVert \phi(s_h,a_h)\rVert_{\hat \Sigma_{j,h}^\dagger}^2] \right ]\right ]+\\
&\quad \frac \beta 2\E\left [\sum_{j=1}^J \sum_{h=1}^{H} \E_{s_h\sim \pi^\ast} \left [\E_{a_h\sim \pi^\ast(\cdot \mid s_h)}[\lVert \phi(s_h,a_h)\rVert_{\hat \Sigma_{j,h}^\dagger}^2] \right ]\right ]+ \O\left (\frac \gamma \beta dH^3 J\right ).
\end{align*}
Different from \Cref{lem:Bias-1 of log-barrier,lem:Bias-1 of variance-reduced}, in that proof, we need to handle the estimation error $\beta^{-1}\lVert (\hat \Sigma_{j,h}-\Sigma_h^{\pi_j})\bar \theta_{j,h}^{\pi_j}(s,a)\rVert_{\hat \Sigma_{j,h}^\dagger}^2$ by the multiplicative bound \Cref{corol:multiplicative error of inverse (no MGR)} instead of the additive one (e.g., \Cref{eq:error of Sigma in variance-reduced}). See \Cref{eq:additional term in bias-1 in linear mdps} for more details.

For the \textsc{Reg-Term}, we again apply the new FTRL lemma \Cref{lem:log-barrier regret bound}, giving the following expression:
\begin{align*}
\E[\textsc{Reg-Term}]&\le \Otil\left (\frac H \eta A+A\sqrt dH^2+\frac{H^3}{\gamma^2}\delta  J\right )\\&+2\eta H^3 \sum_{h=1}^H \E_{s_h\sim \pi^\ast} \left [\sum_{j=1}^J \E_{a_h\sim \pi_j(\cdot \mid s_h)}\left [\lVert \phi(s_h,a_h)\rVert_{\hat \Sigma_{j,h}^\dagger}^2\right ]\right ]\\
&+\sum_{h=1}^H \E_{s_h\sim \pi^\ast} \left [\sum_{j=1}^J \E_{a_h\sim \pi_j(\cdot \mid s_h)}\left [b_j(s,a)\right ]\right ],
\end{align*}
whose formal proof is in \Cref{lem:reg-term in linear MDPs}. In that proof, we need to bound the magnitudes of the bonuses to write $\hat B_j(s_h,a_h)^2\lesssim \frac 1H\hat B_j(s_h,a_h)$. This is done by applying \Cref{lem:bonus magnitude in linear MDP} later in this section.

By summing up all terms and multiplying $W$, we can apply \Cref{lem:dilated bonus} by again using $100\eta H^4\le \beta$. Hence, we get the following by using \Cref{eq:sum of bonus}:
\begin{align*}
\mathcal R_K\le \O(K_0)+\O(\delta_e K)+\Otil\left (\beta dHK+\frac \gamma \beta dH^3 K+\frac H \eta AW+A\sqrt dH^2W+\frac{H^3}{\gamma^2}\delta K\right ),
\end{align*}
where we conditioned on some good event with probability $1-\O(K^{-2}+K\delta)$. Picking $\delta=K^{-3}$, the total regret when the good event does not happen is of order $\O(K^{-2} HK)=o(1)$.

Moreover, by the conditions $\alpha=\frac{\delta_e}{6\beta}$, $M_0\ge \alpha^2 dH^2$, and $N_0\ge 100\frac{M_0^3}{\alpha^2}\log \frac K\delta$, we know that $K_0=M_0N_0=\Otil(\frac{\delta_e^6}{\beta^6} d^4 H^8)$. Meanwhile, by $W=4d\log \frac d\delta \gamma^{-2}$, we know that $W=\Otil(d\gamma^{-2})$. Thus, we have
\begin{equation*}
\mathcal R_K=\Otil\left (\frac{\delta_e^6}{\beta^6} d^4 H^8+\delta_e K+\beta dHK+\frac \gamma \beta dH^3 K+\frac H \eta Ad\gamma^{-2}+Ad^{3/2}H^2\gamma^{-2}\right ).
\end{equation*}

The only conditions are then $\gamma\ge 36\frac{\beta^2}{\delta_e}$, $100\eta H^4\le \beta$, which allows us to set
\begin{equation*}
\delta_e=C(H^{20}Ad^6)^{1/9}K^{-1/9},\beta=\frac{C^2}{36}(H^{13}A^2d^{3}K^{-2})^{1/9}K^{-2/9},\gamma=\frac{C^3}{36}(H^{2}A^1)^{1/3}K^{-1/3},\eta=\frac{C^2}{3600}(H^{-23}A^2d^{3}K^{-2})^{1/9}K^{-2/9},
\end{equation*}
where $C$ is a constant. This ensures $\mathcal R_K=\Otil((H^{20}Ad^6)^{1/9}K^{8/9})$ (note that only the 2nd, 4th, and 5th term have a $K^{8/9}$ dependency, which means all other terms can be ignored when stating the bound).
\end{proof}

\subsection{Bounding the Bias Terms}
\begin{lemma}\label{lem:bias-1 in linear MDPs}
When $W=4d\log \frac d\delta \gamma^{-2}$ and $\gamma<\frac 14$, the \textsc{Bias-1} and \textsc{Bias-2} terms in \Cref{alg:linear MDP without MGR} is bounded by
\begin{align*}
\E[\textsc{Bias-1}]+\E[\textsc{Bias-2}]
&\le\frac \beta 4\E\left [\sum_{j=1}^J \sum_{h=1}^{H} \E_{s_h\sim \pi^\ast} \left [\E_{a_h\sim \pi_j(\cdot \mid s_h)}[\lVert \phi(s_h,a_h)\rVert_{\hat \Sigma_{j,h}^\dagger}^2] \right ]\right ]+\\
&\quad \frac \beta 4\E\left [\sum_{j=1}^J \sum_{h=1}^{H} \E_{s_h\sim \pi^\ast} \left [\E_{a_h\sim \pi^\ast(\cdot \mid s_h)}[\lVert \phi(s_h,a_h)\rVert_{\hat \Sigma_{j,h}^\dagger}^2] \right ]\right ]+ \O\left (\frac \gamma \beta dH^3 K\right ).
\end{align*}
\end{lemma}
\begin{proof}
By direct calculation like \Cref{lem:Bias-1 of log-barrier}, we get the following for any $j\in [J]$, $h\in [H]$, $s\in \mS_h$ and $a\in \mA$ (recall that $\hat \Sigma_{j,h}^\dagger$ only depends on the episodes in $\mT_j$; again, all expectations are only taken to the randomness in epoch $j$)
\begin{align*}
&\quad\E\left [\bar Q_j^{\pi_j}(s,a)-\hat Q_j(s,a)\right ]\\
&=\phi(s,a)^\trans \bar \theta_{j,h}^{\pi_j}-\phi(s,a)^\trans \E\left [\hat \Sigma_{j,h}^\dagger \left (\frac{1}{\lvert \mT_j'\rvert}\sum_{k\in \mT_j'} ((1-Y_k)+Y_kH\mathbbm 1[h=h_k])\phi(s_{k,h},a_{k,h})L_{k,h}\right )\right ]\\
&=\phi(s,a)^\trans \bar \theta_{j,h}^{\pi_j}-\phi(s,a)^\trans \E\left [\hat \Sigma_{j,h}^\dagger \right ]\E \left [\left (\frac{1}{\lvert \mT_j'\rvert}\sum_{k\in \mT_j'} ((1-Y_k)+Y_kH\mathbbm 1[h=h_k])\phi(s_{k,h},a_{k,h})\phi(s_{k,h},a_{k,h})^\trans \theta_{k,h}^{\pi_j}\right )\right ]\\
&=\phi(s,a)^\trans \bar \theta_{j,h}^{\pi_j}-\phi(s,a)^\trans \E\left [\hat \Sigma_{j,h}^\dagger \right ]\Sigma_{h}^{\tilde \pi_j} \bar \theta_{j,h}^{\pi_j}\\
&=\phi(s,a)^\trans (I-\E[\hat \Sigma_{j,h}^\dagger] \Sigma_h^{\tilde \pi_j}) \bar \theta_{j,h}^{\pi_j}.
\end{align*}

By using Cauchy-Schwartz inequality, triangle inequality, and the AM-GM inequality, we get the following:
\begin{align*}
\phi(s,a)^\trans (I-\hat \Sigma_{j,h}^\dagger \Sigma_{h}^{\tilde \pi_j}) \bar \theta_{j,h}^{\pi_j}
&=\phi(s,a)^\trans \hat \Sigma_{j,h}^\dagger (\gamma I+\tilde \Sigma_{j,h}-\Sigma_h^{\tilde\pi_j}) \bar \theta_{j,h}^{\pi_j}\\
&\le \left \lVert \phi(s,a)\right \rVert_{\hat \Sigma_{j,h}^\dagger} \left(\left \lVert (\gamma I)\bar \theta_{j,h}^{\pi_j}\right \rVert_{\hat \Sigma_{j,h}^\dagger}+\left \lVert(\tilde \Sigma_{j,h}-\Sigma_h^{\tilde\pi_j}) \bar \theta_{j,h}^{\pi_j}\right \rVert_{\hat \Sigma_{j,h}^\dagger}\right)\\
&\le \frac \beta 4 \left \lVert \phi(s,a)\right \rVert_{\hat \Sigma_{j,h}^\dagger}^2 + \frac 2\beta \left \lVert \gamma \bar \theta_{j,h}^{\pi_j}(s,a)\right \rVert_{\hat \Sigma_{j,h}^\dagger}^2+\frac 2\beta \left \lVert (\tilde \Sigma_{j,h}-\Sigma_h^{\tilde\pi_j}) \bar \theta_{j,h}^{\pi_j}\right \rVert_{\hat \Sigma_{j,h}^\dagger}^2.
\end{align*}

The first term is the usual bonus term in \Cref{lem:dilated bonus}, while the second term easily translates to the following using the assumption that $\lVert \theta_{k,h}^{\pi_j}\rVert\le \sqrt dH$ for all $k\in \mT_j'$:
\begin{equation*}
\frac 2\beta \left \lVert \gamma \bar \theta_{j,h}^{\pi_j}\right \rVert_{\hat \Sigma_{j,h}^\dagger}^2\le \frac{2\gamma^2}{\beta} \left \lVert (\gamma I+\tilde \Sigma_{j,h})^{-1}\right \rVert_2 dH^2\le 2\frac \gamma \beta dH^2,
\end{equation*}
while the last term translates to the following by algebraic manipulations:
\begin{align}
\frac 2\beta \left \lVert (\tilde \Sigma_{j,h}-\Sigma_h^{\tilde\pi_j}) \bar \theta_{j,h}^{\pi_j}\right \rVert_{\hat \Sigma_{j,h}^\dagger}^2
&= \frac 2\beta \left \lVert (\tilde \Sigma_{j,h}-\Sigma_h^{\tilde\pi_j}) \hat \Sigma_{j,h}^\dagger (\tilde \Sigma_{j,h}-\Sigma_h^{\tilde\pi_j})\right \rVert_2 dH^2\nonumber\\
&=\frac 2\beta \left \lVert (\hat \Sigma_{j,h}^\dagger)^{-1/2} \left (I-(\hat \Sigma_{j,h}^\dagger)^{1/2} (\gamma I+\Sigma_{h}^{\tilde\pi_j}) (\hat \Sigma_{j,h}^\dagger)^{1/2}\right )^2(\hat \Sigma_{j,h}^\dagger)^{-1/2}\right \rVert_2 dH^2,\label{eq:additional term in bias-1 in linear mdps}
\end{align}
where one may expand and check the last step indeed holds.

Using \Cref{corol:multiplicative error of inverse (no MGR)}, the squared-matrix in the middle has its operator norm bounded by $4\gamma$ with high probability (if the good event does not hold, then one can directly bound the last term by matrix Azuma and the operator norm of $\hat \Sigma_{k,h}^\dagger$, giving $\gamma^{-3}$; as this only happens with probability $K^{-3}$, this part contributes $o(K^{-2})$ to the total regret and we thus omit it). Meanwhile, the first and last term both has their operator norms bounded by $\sqrt 2$. Thus,
\begin{equation*}
\frac 2\beta \left \lVert (\tilde \Sigma_{j,h}-\Sigma_h^{\pi_j}) \bar \theta_{j,h}^{\pi_j}\right \rVert_{\hat \Sigma_{j,h}^\dagger}^2 dH^2\le 16\frac \gamma \beta dH^2.
\end{equation*}

In other words, the last term gets absorbed by the second term up to constants. Hence, the conclusion follows by the same argument as \Cref{lem:Bias-1 of log-barrier}.
\end{proof}

\begin{lemma}\label{lem:bias-3 in linear MDPs}
When $W=4d\log \frac d\delta \gamma^{-2}$ and $\gamma<\frac 14$, the \textsc{Bias-3} and \textsc{Bias-4} terms in \Cref{alg:linear MDP without MGR} is bounded by the following when the good event in \Cref{lem:bonus magnitude in linear MDP} holds:
\begin{align*}
\E[\textsc{Bias-3}]+\E[\textsc{Bias-4}]
&\le\frac \beta 4\E\left [\sum_{k=1}^K \sum_{h=1}^{H} \E_{s_h\sim \pi^\ast} \left [\E_{a_h\sim \pi_k(\cdot \mid s_h)}[\lVert \phi(s_h,a_h)\rVert_{\hat \Sigma_{k,h}^\dagger}^2] \right ]\right ]+\\
&\quad \frac \beta 4\E\left [\sum_{k=1}^K \sum_{h=1}^{H} \E_{s_h\sim \pi^\ast} \left [\E_{a_h\sim \pi^\ast(\cdot \mid s_h)}[\lVert \phi(s_h,a_h)\rVert_{\hat \Sigma_{k,h}^\dagger}^2] \right ]\right ]+ \O\left (\frac \gamma \beta dH^3 K\right ).
\end{align*}
\end{lemma}
\begin{proof}
The proof is identical to \Cref{lem:bias-1 in linear MDPs} except for $L_{k,h}$ is replaced by $D_{k,h}$ and $\bar \theta_{j,h}^{\pi_j}$ is replaced by $\Lambda_{j,h}^{\pi_j}$. As we also have $b_j(s,a)\in [0,1]$ by \Cref{lem:bonus magnitude in linear MDP}, the same bound holds.
\end{proof}

\subsection{Bounding \textsc{Reg-Term}}
\begin{lemma}\label{lem:reg-term in linear MDPs}
Assuming $100H\eta H^4\le \beta$ and the good events in \Cref{lem:bonus magnitude in linear MDP} hold. Then the \textsc{Reg-Term} has its expectation bounded by
\begin{align*}
\E[\textsc{Reg-Term}]&\le \Otil\left (\frac H \eta A+A\sqrt dH^2+\frac{H^3}{\gamma^2} \delta J\right )\\
&+2\eta H^3 \sum_{h=1}^H \E_{s_h\sim \pi^\ast} \left [\sum_{j=1}^J \E_{a_h\sim \pi_j(\cdot \mid s_h)}\left [\lVert \phi(s_h,a_h)\rVert_{\hat \Sigma_{j,h}^\dagger}^2\right ]\right ]\\
&+\sum_{h=1}^H \E_{s_h\sim \pi^\ast} \left [\sum_{j=1}^J \E_{a_h\sim \pi_j(\cdot \mid s_h)}\left [b_j(s,a)\right ]\right ].
\end{align*}
\end{lemma}
\begin{proof}
The proof generally follows from \Cref{lem:Reg-Term of log-barrier}, except for some tiny differences due to epoching.

Using \Cref{lem:log-barrier regret bound}, we get the following for all $j\in [J]$, $s\in \mS_h\cap \mathcal K$ (where $h\in [H]$), and $\tilde \pi^\ast\in \Pi$:
\begin{align*}
&\quad \sum_{j=1}^J\sum_{a\in \mA} (\pi_j(a\mid s)-\pi^\ast(a\mid s))(\hat Q_j(s,a)-\hat B_j(s,a))\\
&\le \frac{\Psi(\tilde \pi^\ast(\cdot \mid s))-\Psi(\pi_1(\cdot \mid s))}{\eta}+\\
&\quad \sum_{j=1}^J\sum_{a\in \mA} (\tilde \pi^\ast(a\mid s)-\pi^\ast(a\mid s))(\hat Q_j(s,a)-\hat B_j(s,a))+\\
&\quad \eta \sum_{j=1}^J \sum_{a\in \mA} \pi_j(a\mid s) (\hat Q_j(s,a)-\hat B_j(s,a))^2.
\end{align*}

By picking $\tilde \pi^\ast(a\mid s)=(1-AJ^{-1})\pi^\ast(a\mid s)+J^{-1}$ as \Cref{lem:Reg-Term of log-barrier}, the first term is bounded by $\eta^{-1} A \ln J$.
Meanwhile, the second term is bounded by
\begin{align*}
&\quad \sum_{j=1}^J\sum_{a\in \mA} (\tilde \pi^\ast(a\mid s)-\pi^\ast(a\mid s))(\hat Q_j(s,a)-\hat B_j(s,a))\\
&=\sum_{j=1}^J \sum_{a\in \mA} (-AJ^{-1}\pi^\ast(a\mid s)+J^{-1})(\hat Q_j(s,a)-\hat B_j(s,a)).
\end{align*}

Like what we did in \Cref{lem:Reg-Term of log-barrier}, we consider the expected difference between $\hat Q_j$ and $\bar Q_j^{\pi_j}$:
\begin{equation*}
\E[\hat Q_j(s,a)-\bar Q_j^{\pi_j}(s,a)]=\phi(s,a)^\trans (I-\E[\hat \Sigma_{j,h}^\dagger\Sigma_h^{\tilde\pi_j}])\bar \theta_{j,h}^{\pi_j}\le 2\sqrt d H,
\end{equation*}
where the first inequality is due to \Cref{lem:bias-1 in linear MDPs} and the second one uses \Cref{lem:multiplicative error (no MGR)}. Moreover, according to \Cref{lem:bonus magnitude in linear MDP}, the same bound also holds for $\E[\hat B_j(s,a)-\bar B_j^{\pi_j}(s,a)]$.

Hence, after taking expectations on both sides, we know that
\begin{align*}
&\quad \E\left [\sum_{j=1}^J\sum_{a\in \mA} (\tilde \pi^\ast(a\mid s)-\pi^\ast(a\mid s))(\hat Q_j(s,a)-\hat B_j(s,a))\right ]\\
&\le \sum_{j=1}^J \sum_{a\in \mA} (\lvert -AJ^{-1}\pi^\ast(a\mid s)\rvert+\lvert J^{-1}\rvert)\lvert \E[\hat Q_j(s,a)-\hat B_j(s,a)]\rvert\\
&\le J \times 2AJ^{-1}\times 4\sqrt dH=\O\left (A\sqrt dH\right ).
\end{align*}

Then consider the last term. We still write $(\hat Q_j(s,a)-\hat B_j(s,a))^2\le 2\hat Q_j(s,a)^2+2\hat B_j(s,a)^2$, which can be calculated as follows:
\begin{align*}
\E[\hat Q_j(s,a)^2]&\le H^2 \E\left [\frac{1}{\lvert \mT_j'\rvert}\sum_{k\in \mT_j'} \phi(s,a)^\trans \hat \Sigma_{j,h}^\dagger \left (((1-Y_k)+Y_kH\mathbbm 1[h=h_k])^2 \phi(s_{k,h},a_{k,h})\phi(s_{k,h},a_{k,h})^\trans\right )\hat \Sigma_{j,h}^\dagger \phi(s,a)\right ]\\
&=H^2\E\left [\phi(s,a)^\trans \hat \Sigma_{j,h}^\dagger \left ((1-\delta_e) \Sigma_{h}^{\pi_j}+H\delta_e \Sigma_h^{\text{cov}}\right )\hat \Sigma_{j,h}^\dagger \phi(s,a)\right ]\\
&\le H^3 \E\left [\phi(s,a)^\trans \hat \Sigma_{j,h}^\dagger \Sigma_{h}^{\tilde \pi_j} \hat \Sigma_{j,h}^\dagger \phi(s,a)\right ].
\end{align*}

Then we use \Cref{corol:multiplicative error of inverse (no MGR)}. If the good event does not happen, then this term is bounded by $\O(H^3 \gamma^{-2} \delta)$. Otherwise, it can be written as $2H^3 \E\left [\lVert \phi(s,a)\rVert_{\hat \Sigma_{j,h}^\dagger}^2\right ]$. Then we consider the estimated dilated bonus term:
\begin{align*}
\E[\hat B_j(s,a)^2]&\le 2\E[b_j(s,a)^2]+2\E[(\phi(s,a)^\trans \hat \Lambda_{j,h})^2].
\end{align*}

According to \Cref{lem:bonus magnitude in linear MDP} (whose failure only contributes in total $o(1)$ regret as we pick $\delta=K^{-3}$), we know $b_j(s,a)\le 1$, which means $D_{k,h}\le 3H$. Thus, the second term is also bounded by $\O(H^3 \E[\lVert \phi(s,a)\rVert_{\hat \Sigma_{j,h}^\dagger}^2]+H^3\gamma^{-2}\delta)$. Moreover, as $b_j(s,a)\le 1$, the first term is bounded by $\E[b_j(s,a)]$. Thus, by definition of $b_j(s,a)$, we get
\begin{equation*}
\E[\hat B_j(s,a)^2]\le \O\left (\frac{H^3}{\beta}\right )b_j(s,a)+\O(H^3\gamma^{-2}\delta)).
\end{equation*}

Putting everything together gives
\begin{align*}
\E[\textsc{Reg-Term}]&\le \Otil\left (\frac H \eta A+A\sqrt dH^2+\frac{H^3}{\gamma^2} \delta J\right )\\
&+2\eta H^3 \sum_{h=1}^H \E_{s_h\sim \pi^\ast} \left [\sum_{j=1}^J \E_{a_h\sim \pi_j(\cdot \mid s_h)}\left [\lVert \phi(s_h,a_h)\rVert_{\hat \Sigma_{j,h}^\dagger}^2\right ]\right ]\\
&+\frac{100\eta H^3}{\beta} \sum_{h=1}^H \E_{s_h\sim \pi^\ast} \left [\sum_{j=1}^J \E_{a_h\sim \pi_j(\cdot \mid s_h)}\left [b_j(s,a)\right ]\right ].
\end{align*}

This then translates to our conclusion using the condition that $100H^4\eta\le \beta$.
\end{proof}

\subsection{Bounding the Magnitudes of Bonuses}
\begin{lemma}\label{lem:bonus magnitude in linear MDP}
Let $\alpha=\frac{\delta_e}{6\beta}$, $M_0\ge \alpha^2 dH^2$, $N_0\ge 100 \frac{M_0^3}{\alpha^2} \log \frac K\delta$, and $\gamma\ge 36\frac{\beta^2}{\delta_e}$. Then with probability $1-(K^{-2}+K\delta)$, $b_j(s,a)=\mathbbm 1[(s,a)\in \mathcal K]\times \beta (\lVert \phi(s,a)\rVert_{\hat \Sigma_{j,h}^\dagger}^2+\E_{a'\sim \pi_j(\cdot \mid s)}[\lVert \phi(s,a')\rVert_{\hat \Sigma_{j,h}^\dagger}])\le 1$ for all $j\in [J]$, $h\in [H]$, $s\in \mS_h$ and $a\in \mA$.
\end{lemma}
The proof is similar to, but different from Lemma F.1 of \citet{luo2021policy}. Here, we use the alternative for the MGR procedure. We also refine the analysis on $M_0^{-1}+\Sigma_h^{\text{cov}}\approx \hat \Sigma_h^{\text{cov}}$ for a smaller $N_0$, which is critical for our new regret bound.
\begin{proof}
It suffices to show that $\beta \lVert \phi(s,a')\rVert_{\hat \Sigma_{k,h}^\dagger}^2\le \frac 12$ for any $s\in \mathcal K$ and $a'\in \mA$. Firstly, we have the following with high probability because $\hat \Sigma_{j,h}^\dagger$ is a multiplicative approximation (i.e., \Cref{corol:multiplicative error of inverse (no MGR)}):
\begin{equation*}
\beta \lVert \phi(s,a')\rVert_{\hat \Sigma_{j,h}^\dagger}\le 2\beta \lVert \phi(s,a')\rVert_{(\gamma I+\Sigma_h^{\tilde \pi_j})^{-1}}^2
\le 2\frac{\beta}{\delta_e} \lVert \phi(s,a')\rVert_{(\frac{\gamma}{\delta_e} I+\Sigma_h^{\text{cov}})^{-1}}^2\le 2\frac{\beta}{\delta_e} \lVert \phi(s,a')\rVert_{(M_0^{-1} I+\Sigma_h^{\text{cov}})^{-1}}^2,
\end{equation*}
where $\Sigma_h^{\text{cov}}$ is the short-hand notation of $\Sigma_h^{\pi_\text{cov}}$ and the last step uses the fact that $\frac{\gamma}{\delta_e}M_0\ge \frac{\gamma}{\delta_e}\frac{\delta_e^2}{36\beta^2}\ge 1$. Moreover, we argue that $M_0^{-1}+\Sigma_h^{\text{cov}}\approx \hat \Sigma_h^{\text{cov}}$. By definition of $\hat \Sigma_h^{\text{cov}}$ (see \Cref{alg:PolicyCover}):
\begin{equation*}
\hat \Sigma_h^{\text{cov}}=\frac{1}{M_0}I+\frac{1}{M_0N_0}\sum_{m=1}^{M_0}\sum_{n=1}^{N_0}\phi(s_{m,n,h},a_{m,n,h})\phi(s_{m,n,h},a_{m,n,h})^\trans,
\end{equation*}
we can apply the Matrix Azuma inequality (\Cref{lem:matrix azuma}) to ensure that (where we simply pick $X_{mN_0+n}=\phi(s_{m,n,h},a_{m,n,h})\phi(s_{m,n,a},a_{m,n,h})^\trans-\Sigma_h^{\pi_m}$ and $A_{mN_0+n}=I\succeq X_{mN_0+n}$; there are in total $M_0N_0$ matrices):
\begin{equation*}
\left \lVert \hat \Sigma_h^{\text{cov}}-\frac{1}{M_0}I-\Sigma_h^{\text{cov}}\right \rVert_2\le \sqrt{\frac{8}{M_0N_0}\log \frac d\delta}.
\end{equation*}

The rest follows the proof of the original lemma \citep[Lemma F.1]{luo2021policy}. By following the proof of Theorem 2.1 of \citet{meng2010optimal}, we can conclude that
\begin{equation*}
\left \lVert \left (\hat \Sigma_h^{\text{cov}}\right )^{-1}-\left (\frac{1}{M_0}I-\Sigma_h^{\text{cov}}\right )^{-1}\right \rVert_2
\le M_0^2 \left \lVert \hat \Sigma_h^{\text{cov}}-\frac{1}{M_0}I-\Sigma_h^{\text{cov}}\right \rVert_2\le M_0^2 \sqrt{\frac{8}{M_0N_0}\log \frac d\delta}= \sqrt{\frac{8M_0^3}{N_0}\log \frac d\delta}.
\end{equation*}

Therefore, we only need $M_0=\O(N_0^3)$ to make it bounded by $\frac \alpha 2$. Consequently, as $\lVert \phi(s,a')\rVert_2\le 1$,
\begin{align*}
&\quad \beta \lVert \phi(s,a')\rVert_{\hat \Sigma_{j,h}^\dagger}
\le 2\frac{\beta}{\delta_e} \lVert \phi(s,a')\rVert_{(M_0^{-1}I+\Sigma_h^{\text{cov}})^{-1}}\\
&\le 2\frac{\beta}{\delta_e}\left (\lVert \phi(s,a')\rVert_{(\hat \Sigma_h^{\text{cov}})^{-1}}^2+\left \lVert \left (\hat \Sigma_h^{\text{cov}}\right )^{-1}-\left (\frac{1}{M_0}I-\Sigma_h^{\text{cov}}\right )^{-1}\right \rVert_2\right )\\
&\le 2\frac{\beta}{\delta_e} \left (\alpha+\frac \alpha 2\right )\le \frac 12,
\end{align*}
by the condition that $s\in \mathcal K$ and our choice of $\alpha$.
\end{proof}

%% file: main.bbl
\begin{thebibliography}{35}
\providecommand{\natexlab}[1]{#1}
\providecommand{\url}[1]{\texttt{#1}}
\expandafter\ifx\csname urlstyle\endcsname\relax
  \providecommand{\doi}[1]{doi: #1}\else
  \providecommand{\doi}{doi: \begingroup \urlstyle{rm}\Url}\fi

\bibitem[Abbasi-Yadkori et~al.(2019)Abbasi-Yadkori, Bartlett, Bhatia, Lazic,
  Szepesvari, and Weisz]{abbasi2019politex}
Abbasi-Yadkori, Y., Bartlett, P., Bhatia, K., Lazic, N., Szepesvari, C., and
  Weisz, G.
\newblock Politex: Regret bounds for policy iteration using expert prediction.
\newblock In \emph{International Conference on Machine Learning}, pp.\
  3692--3702. PMLR, 2019.

\bibitem[Agarwal et~al.(2020)Agarwal, Henaff, Kakade, and Sun]{agarwal2020pc}
Agarwal, A., Henaff, M., Kakade, S., and Sun, W.
\newblock Pc-pg: Policy cover directed exploration for provable policy gradient
  learning.
\newblock \emph{Advances in neural information processing systems},
  33:\penalty0 13399--13412, 2020.

\bibitem[Auer et~al.(2002)Auer, Cesa-Bianchi, Freund, and
  Schapire]{auer2002nonstochastic}
Auer, P., Cesa-Bianchi, N., Freund, Y., and Schapire, R.~E.
\newblock The nonstochastic multiarmed bandit problem.
\newblock \emph{SIAM journal on computing}, 32\penalty0 (1):\penalty0 48--77,
  2002.

\bibitem[Bubeck et~al.(2012)Bubeck, Cesa-Bianchi, and
  Kakade]{bubeck2012towards}
Bubeck, S., Cesa-Bianchi, N., and Kakade, S.~M.
\newblock Towards minimax policies for online linear optimization with bandit
  feedback.
\newblock In \emph{Conference on Learning Theory}, pp.\  41--1. JMLR Workshop
  and Conference Proceedings, 2012.

\bibitem[Cai et~al.(2020)Cai, Yang, Jin, and Wang]{cai2020provably}
Cai, Q., Yang, Z., Jin, C., and Wang, Z.
\newblock Provably efficient exploration in policy optimization.
\newblock In \emph{International Conference on Machine Learning}, pp.\
  1283--1294. PMLR, 2020.

\bibitem[Dani et~al.(2008)Dani, Kakade, and Hayes]{dani2008price}
Dani, V., Kakade, S.~M., and Hayes, T.
\newblock The price of bandit information for online optimization.
\newblock \emph{Advances in Neural Information Processing Systems}, 20, 2008.

\bibitem[Foster et~al.(2016)Foster, Li, Lykouris, Sridharan, and
  Tardos]{foster2016learning}
Foster, D.~J., Li, Z., Lykouris, T., Sridharan, K., and Tardos, E.
\newblock Learning in games: Robustness of fast convergence.
\newblock \emph{Advances in Neural Information Processing Systems}, 29, 2016.

\bibitem[He et~al.(2022)He, Zhou, and Gu]{he2022near}
He, J., Zhou, D., and Gu, Q.
\newblock Near-optimal policy optimization algorithms for learning adversarial
  linear mixture mdps.
\newblock In \emph{International Conference on Artificial Intelligence and
  Statistics}, pp.\  4259--4280. PMLR, 2022.

\bibitem[Ito(2021)]{ito2021parameter}
Ito, S.
\newblock Parameter-free multi-armed bandit algorithms with hybrid
  data-dependent regret bounds.
\newblock In \emph{Conference on Learning Theory}, pp.\  2552--2583. PMLR,
  2021.

\bibitem[Jin et~al.(2020{\natexlab{a}})Jin, Jin, Luo, Sra, and
  Yu]{jin2020learning}
Jin, C., Jin, T., Luo, H., Sra, S., and Yu, T.
\newblock Learning adversarial markov decision processes with bandit feedback
  and unknown transition.
\newblock In \emph{International Conference on Machine Learning}, pp.\
  4860--4869. PMLR, 2020{\natexlab{a}}.

\bibitem[Jin et~al.(2020{\natexlab{b}})Jin, Yang, Wang, and
  Jordan]{jin2020provably}
Jin, C., Yang, Z., Wang, Z., and Jordan, M.~I.
\newblock Provably efficient reinforcement learning with linear function
  approximation.
\newblock In \emph{Conference on Learning Theory}, pp.\  2137--2143. PMLR,
  2020{\natexlab{b}}.

\bibitem[Kakade \& Langford(2002)Kakade and Langford]{kakade2002approximately}
Kakade, S. and Langford, J.
\newblock Approximately optimal approximate reinforcement learning.
\newblock In \emph{In Proc. 19th International Conference on Machine Learning}.
  Citeseer, 2002.

\bibitem[Kong et~al.(2023)Kong, Zhang, Wang, and Li]{kong2023improved}
Kong, F., Zhang, X., Wang, B., and Li, S.
\newblock Improved regret bounds for linear adversarial mdps via linear
  optimization.
\newblock \emph{arXiv preprint arXiv:2302.06834}, 2023.

\bibitem[Lancewicki et~al.(2023)Lancewicki, Rosenberg, and
  Sotnikov]{lancewicki2023delay}
Lancewicki, T., Rosenberg, A., and Sotnikov, D.
\newblock Delay-adapted policy optimization and improved regret for adversarial
  mdp with delayed bandit feedback.
\newblock \emph{arXiv preprint arXiv:2305.07911}, 2023.

\bibitem[Lattimore \& Szepesv{\'a}ri(2020)Lattimore and
  Szepesv{\'a}ri]{lattimore2020bandit}
Lattimore, T. and Szepesv{\'a}ri, C.
\newblock \emph{Bandit algorithms}.
\newblock Cambridge University Press, 2020.

\bibitem[Luo et~al.(2021{\natexlab{a}})Luo, Wei, and Lee]{luo2021policy}
Luo, H., Wei, C.-Y., and Lee, C.-W.
\newblock Policy optimization in adversarial mdps: Improved exploration via
  dilated bonuses.
\newblock \emph{arXiv preprint arXiv:2107.08346}, 2021{\natexlab{a}}.

\bibitem[Luo et~al.(2021{\natexlab{b}})Luo, Wei, and
  Lee]{luo2021policy_nipsver}
Luo, H., Wei, C.-Y., and Lee, C.-W.
\newblock Policy optimization in adversarial mdps: Improved exploration via
  dilated bonuses.
\newblock \emph{Advances in Neural Information Processing Systems},
  34:\penalty0 22931--22942, 2021{\natexlab{b}}.

\bibitem[Meng \& Zheng(2010)Meng and Zheng]{meng2010optimal}
Meng, L. and Zheng, B.
\newblock The optimal perturbation bounds of the moore--penrose inverse under
  the frobenius norm.
\newblock \emph{Linear algebra and its applications}, 432\penalty0
  (4):\penalty0 956--963, 2010.

\bibitem[Neu \& Olkhovskaya(2020)Neu and Olkhovskaya]{neu2020efficient}
Neu, G. and Olkhovskaya, J.
\newblock Efficient and robust algorithms for adversarial linear contextual
  bandits.
\newblock In \emph{Conference on Learning Theory}, pp.\  3049--3068. PMLR,
  2020.

\bibitem[Neu \& Olkhovskaya(2021)Neu and Olkhovskaya]{neu2021online}
Neu, G. and Olkhovskaya, J.
\newblock Online learning in mdps with linear function approximation and bandit
  feedback.
\newblock \emph{Advances in Neural Information Processing Systems},
  34:\penalty0 10407--10417, 2021.

\bibitem[Putta \& Agrawal(2022)Putta and Agrawal]{putta2022scale}
Putta, S.~R. and Agrawal, S.
\newblock Scale-free adversarial multi armed bandits.
\newblock In \emph{International Conference on Algorithmic Learning Theory},
  pp.\  910--930. PMLR, 2022.

\bibitem[Rosenberg \& Mansour(2019)Rosenberg and Mansour]{rosenberg2019online}
Rosenberg, A. and Mansour, Y.
\newblock Online convex optimization in adversarial markov decision processes.
\newblock In \emph{International Conference on Machine Learning}, pp.\
  5478--5486. PMLR, 2019.

\bibitem[Shani et~al.(2020)Shani, Efroni, Rosenberg, and
  Mannor]{shani2020optimistic}
Shani, L., Efroni, Y., Rosenberg, A., and Mannor, S.
\newblock Optimistic policy optimization with bandit feedback.
\newblock In \emph{International Conference on Machine Learning}, pp.\
  8604--8613. PMLR, 2020.

\bibitem[Sherman et~al.(2023)Sherman, Koren, and Mansour]{sherman2023improved}
Sherman, U., Koren, T., and Mansour, Y.
\newblock Improved regret for efficient online reinforcement learning with
  linear function approximation.
\newblock \emph{arXiv preprint arXiv:2301.13087}, 2023.

\bibitem[Tropp(2012)]{tropp2012user}
Tropp, J.~A.
\newblock User-friendly tail bounds for sums of random matrices.
\newblock \emph{Foundations of computational mathematics}, 12\penalty0
  (4):\penalty0 389--434, 2012.

\bibitem[Wagenmaker \& Jamieson(2022)Wagenmaker and
  Jamieson]{wagenmaker2022instance}
Wagenmaker, A. and Jamieson, K.~G.
\newblock Instance-dependent near-optimal policy identification in linear mdps
  via online experiment design.
\newblock \emph{Advances in Neural Information Processing Systems},
  35:\penalty0 5968--5981, 2022.

\bibitem[Wang et~al.(2020)Wang, Du, Yang, and Salakhutdinov]{wang2020reward}
Wang, R., Du, S.~S., Yang, L., and Salakhutdinov, R.~R.
\newblock On reward-free reinforcement learning with linear function
  approximation.
\newblock \emph{Advances in neural information processing systems},
  33:\penalty0 17816--17826, 2020.

\bibitem[Wei \& Luo(2018)Wei and Luo]{wei2018more}
Wei, C.-Y. and Luo, H.
\newblock More adaptive algorithms for adversarial bandits.
\newblock In \emph{Conference On Learning Theory}, pp.\  1263--1291. PMLR,
  2018.

\bibitem[Wei et~al.(2021)Wei, Jahromi, Luo, and Jain]{wei2021learning}
Wei, C.-Y., Jahromi, M.~J., Luo, H., and Jain, R.
\newblock Learning infinite-horizon average-reward mdps with linear function
  approximation.
\newblock In \emph{International Conference on Artificial Intelligence and
  Statistics}, pp.\  3007--3015. PMLR, 2021.

\bibitem[Yang \& Wang(2020)Yang and Wang]{yang2020reinforcement}
Yang, L. and Wang, M.
\newblock Reinforcement learning in feature space: Matrix bandit, kernels, and
  regret bound.
\newblock In \emph{International Conference on Machine Learning}, pp.\
  10746--10756. PMLR, 2020.

\bibitem[Zanette et~al.(2021)Zanette, Cheng, and
  Agarwal]{zanette2021cautiously}
Zanette, A., Cheng, C.-A., and Agarwal, A.
\newblock Cautiously optimistic policy optimization and exploration with linear
  function approximation.
\newblock In \emph{Conference on Learning Theory}, pp.\  4473--4525. PMLR,
  2021.

\bibitem[Zheng et~al.(2019)Zheng, Luo, Diakonikolas, and
  Wang]{zheng2019equipping}
Zheng, K., Luo, H., Diakonikolas, I., and Wang, L.
\newblock Equipping experts/bandits with long-term memory.
\newblock \emph{Advances in Neural Information Processing Systems}, 32, 2019.

\bibitem[Zhou et~al.(2021)Zhou, He, and Gu]{zhou2021provably}
Zhou, D., He, J., and Gu, Q.
\newblock Provably efficient reinforcement learning for discounted mdps with
  feature mapping.
\newblock In \emph{International Conference on Machine Learning}, pp.\
  12793--12802. PMLR, 2021.

\bibitem[Zimin \& Neu(2013)Zimin and Neu]{zimin2013online}
Zimin, A. and Neu, G.
\newblock Online learning in episodic markovian decision processes by relative
  entropy policy search.
\newblock \emph{Advances in neural information processing systems}, 26, 2013.

\bibitem[Zimmert \& Lattimore(2022)Zimmert and Lattimore]{zimmert2022return}
Zimmert, J. and Lattimore, T.
\newblock Return of the bias: Almost minimax optimal high probability bounds
  for adversarial linear bandits.
\newblock In \emph{Conference on Learning Theory}, pp.\  3285--3312. PMLR,
  2022.

\end{thebibliography}
